\documentclass[12pt]{UOthesis}

\usepackage{fancybox}
\usepackage{shadow}
\usepackage[all]{xy}
\usepackage{url}
\usepackage{rotating}
\usepackage{makeidx}
\makeindex

\allowdisplaybreaks

\NoDedications
\NoResume
\NoListoftables
\NoListoffigures
\NoListofsymbols
\NoIntro

\setUOname{Damjan Kalajdzievski}         
\setUOcpryear{2012}               
\setUOtitle{Measurability Aspects of the Compactness Theorem for Sample Compression Schemes}  

\setUOsupone{Vladimir Pestov}                

\msc               

\setUOabstract{
~~

In 1998, it was proved by Ben-David and Litman that a concept space has a sample compression scheme of size $d$ if and only if every finite subspace has a sample compression scheme of size $d$. In the compactness theorem, measurability of the hypotheses of the created sample compression scheme is not guaranteed; at the same time measurability of the hypotheses is a necessary condition for learnability. 

In this thesis we discuss when a sample compression scheme, created from compression schemes on finite subspaces via the compactness theorem, have measurable hypotheses. We show that if $X$ is a standard Borel space with a $d$-maximum and universally separable concept class $\m{C}$, then $(X,\CC)$ has a sample compression scheme of size $d$ with universally Borel measurable hypotheses. Additionally we introduce a new variant of compression scheme called a copy sample compression scheme.

}

\setUOthanks{I would like to thank my supervisor Dr. Pestov for all that he has taught me and for his time and support. I am also grateful to have had the Ontario Graduate Scholarship as financial support, and the financial support of the Department of Mathematics at the University of Ottawa. I would like to express my thanks to both thesis examiners, Dr. J. Levy and Dr. S. Zilles, for their careful reading of the thesis, and suggesting numerous improvements.}

{\theorembodyfont{\rm}
  \newtheorem{rmk}[theo]{Remark}
  \newtheorem{egg}[theo]{Example}
  \newtheorem{notation}[theo]{Notation}
}
\newcommand{\RR}{\mathbb{R}}
\newcommand{\NN}{\mathbb{N}}
\newcommand{\m}[1]{\mathcal{#1}}
\newcommand{\M}[1]{\mathcal{#1}}

\newcommand{\CC}{\mathcal{C}}
\newcommand{\LL}{\mathcal{L}}
\newcommand{\FF}{\mathcal{F}}
\newcommand{\HH}{\mathcal{H}}
\newcommand{\hh}{\mathfrak{H}}

\newcommand{\A}{\mathfrak{A}}
\newcommand{\s}{\subseteq}
\newcommand{\0}{\sigma}
\newcommand{\e}{\varepsilon}

\DeclareMathOperator{\dom}{dom}
\DeclareMathOperator{\range}{range}

\DeclareMathOperator{\vc}{VC}

\begin{document}


\nonumchapter{Introduction}
~~

The common context for Statistical Learning Theory is the setting of a concept space. A concept space is a set $X$ and a family $\CC$ of subsets of $X$ called concepts. The goal of Statistical Learning Theory is to be able to, given a hidden concept $C$ in $\CC$, ``learn" the concept by sampling finite subsets of $X$ labelled according to their membership in $C$. Learning is usually provided by a function or algorithm which takes a labelled sample as an input, and outputs a subset of $X$ called a hypothesis. The function or algorithm is said to be consistent if the hypothesis, for a given sample labelled according to a concept, agrees with the labelling on this sample. The distance between the hypothesis and the target concept can be quantified in varying ways depending on the model of ``learning".

The PAC, or ``Probably Approximately Correct", model for learning was introduced by Valiant in the 80's \cite{Valiant:1984:TL:800057.808710}. A concept space $(X,\CC)$ is PAC learnable with respect to a family $\m{P}$ of probability measures on $X$ if there exists a function, called a ``learning rule", mapping labelled samples in $X$ to subsets of $X$, where for any $0<\e\leq1,\ 0<\delta\leq1$, there is a positive integer $m$ such that given any $C\in\CC$, $P\in\m{P}$, the probability (according to $P$) of sample of size greater than $m$ being mapped to a hypothesis with error from $C$ greater than $\e$, is less than $\delta$ (error of a hypothesis from a concept is meant to be the probability (according to $P$) of their symmetric difference). In particular if $\m{P}$ is the family of all probability measures on $X$, one speaks of distribution-free PAC learnability, and this is the context that will be of exclusive interest to us.

Learnability is intimately related to the concept of VC dimension, or ``Vapnik-Chervonenkis Dimension", introduced by Vapnik and Chervonenkis in \cite{vapnik:264}. VC dimension is a parameter quantifying the combinatorial complexity of a concept space defined from the idea of a subset of $X$ being "shattered" by concepts. $A\s X$ is shattered by $\CC$ if any subset of $A$ is equal to some concept in $\CC$ intersected with $A$; the VC dimension of a concept space is the supremum of the cardinalities of all finite subsets which are shattered by $\CC$. In 1986 Blumer et al. related VC dimension to PAC learnability by proving that (given a measurability condition called ``well behaved") a concept space has finite VC dimension if and only if it is PAC learnable if and only if any consistent learning rule which provides concepts as hypotheses, learns the concept space with sample sizes (sample complexity) bounded above by a formula involving the VC dimension of the concept space. The bounds mentioned are improved by Shawe-Taylor et al. in \cite{Shawe-taylor93boundingsample} who assumed stronger measurability conditions.
 
A natural class of consistent learning algorithms are sample compression schemes which were introduced by Littlestone and Warmuth in 1986 \cite{Littlestone86relatingdata}. A sample compression scheme of size $d$ maps subsets (labelled or unlabelled depending on the variant) of $X$ of size at most $d$, to hypotheses, such that each sample labelled according to a concept has a subsample of size at most $d$ which is mapped to a hypothesis agreeing with the concept on the initial sample. A sample compression scheme of size $d$ can be thought to save every sample labelled according to a concept to some subsample of size at most $d$. Also in \cite{Littlestone86relatingdata} it is shown that, given measurability conditions of the sample compression scheme and concept space, every sample compression scheme is a PAC learning rule with sample complexity bounded above by a formula involving the size of the sample compression scheme (the bounds are due to Floyd and Warmuth in \cite{Floyd95samplecompression}). For a concept space of VC dimension $d$, a sample compression scheme of size $d$ learns with bounds on sample complexity better than that of the bounds for general consistent learning rules provided by the VC dimension (illustrated in Figure 3.1 in chapter 3). In chapter 2 we define our own variant of sample compression scheme called ``copy sample compression schemes", of which sample compression schemes are a special case. A copy sample compression scheme can be thought of as an algorithm which checks sample compression schemes of varying sizes, and for different concept classes, and picks a hypothesis for the concepts in any of these different concept classes. Copy sample compression schemes also add other flexibilities as will be exhibited in chapter 3, and may have better bounds in some instances for sample complexity than that of sample compression schemes.

A natural open question posed in \cite{Littlestone86relatingdata} is whether or not a concept space with VC dimension $d$ also has a sample compression scheme of size $O(d)$. The question currently remains open. The existence of an unlabelled sample compression scheme of size $d$ does imply that the concept space has VC dimension at most $d$, and in some cases like that of maximum classes, VC dimension $d$ is enough to provide a sample compression scheme of size $d$ \cite{Kuzmin06unlabeledcompression}. 

In \cite{Ben-david98combinatorialvariability}, it was proved by Ben-David and Litman, using a proof based on the compactness theorem of predicate logic, that a concept space $(X,\CC)$ has a sample compression scheme of size $d$ if and only if every finite subspace has a sample compression scheme of size $d$ ($(Y,\CC')$ is a finite subspace if $Y\s X$ is finite and $\CC'=\{C\cap Y:C\in\CC\}$). We provide a different, and technically simpler, proof of this using an approach with ultralimits normally used in Analysis. In either case the result, named the ``compactness theorem" for sample compression schemes, did not take any measurability considerations into account, and as our example in chapter 4 shows the resulting hypotheses need not be measurable; chapter 4 of this thesis explores when the compression scheme resulting from the compactness theorem has measurable hypotheses. Perhaps the most useful of these results is that, when $X$ is a standard Borel space with a $d$-maximum and universally separable concept class $\m{C}$, then $(X,\CC)$ has a sample compression scheme of size $d$ with universally Borel measurable hypotheses. In the appendix B.1 we also collect differing measurability conditions rules from varying relevant papers. These conditions are defined to be utilized in proving different bounds for sample complexity of consistent learning rules and for sample complexity of sample compression schemes in these various papers.

This Thesis starts in chapter 1 by outlining VC dimension and important concepts such as the concepts of maximality due to \cite{welzl87rangespaces} and embeddability due to \cite{Ben-david98combinatorialvariability}. In chapter 2 sample compression schemes and their variants are presented and discussed. In Chapter 3 we introduce PAC learnability, and investigate sample complexity for spaces with finite VC dimension and sample complexity for spaces with sample compression schemes or copy sample compression schemes. Finally in chapter 4 we investigate measurability of hypotheses from sample compression schemes generated by the compactness theorem. We also include two appendices, with appendix A consisting of some preliminaries, and appendix B consisting of some excluded proofs and a list of varying technical measurability conditions which are referenced throughout the thesis.








\cleardoublepage

\chapter[Vapnik-Chervonenkis Dimension]{Vapnik-Chervonenkis Dimension}

\section{Vapnik-Chervonenkis Dimension}
~~~

We begin with the definitions of a concept space and the VC dimension associated to a concept space.

\begin{defn}
A {\bfseries concept space}\index{Concept space} is a pair $(X,\mathcal{C})$ consisting of a set $X$ equipped with a set $\mathcal{C}$ of subsets of $X$. $X$ is referred to as the {\bfseries domain}\index{Domain}, and $\CC$ is referred to as the {\bfseries concept class}\index{Concept class}. For a subset $A$ of $X$, denote $$\CC\sqcap A=\{C\cap A : C\in \mathcal{C}\},$$ and we say that $(Y,\mathcal{C}')$ is a {\bfseries subspace}\index{Subspace} of $(X,\mathcal{C})$ if $Y\subseteq X$ and $\mathcal{C}'=\CC\sqcap Y$.
\end{defn}

\begin{defn}[\cite{vapnik:264}]
We say that a subset $A$ of $X$ is {\bfseries shattered}\index{Shattered} by $\mathcal{C}$ if 
$\CC\sqcap A=2^A$.
\end{defn}

\begin{defn}[\cite{vapnik:264}]
The {\bfseries Vapnik-Chervonenkis dimension}\index{VC dimension} or {\bfseries VC-dimension} of $(X,\mathcal{C})$ (denoted $\vc(X,\mathcal{C})$, or $\vc(\mathcal{C})$ when $X$ is understood) is $$\vc(\mathcal{C})=\sup\{ |A| : A\subseteq X, \text{A is finite, A is shattered by } \mathcal{C}\}.$$ In particular if the value is infinite, we say $\vc(\mathcal{C})=\infty$.
\end{defn}

The following are some elementary or well known examples of VC dimension which can be found in every text on statistical learning.
\begin{egg}
Let $X$ be any infinite set and $\CC=2^X$, Then clearly $\vc(X,\CC)=\infty$ because every (finite) $A\s X$ has $\CC\sqcap A=2^A$ and so $A$ is shattered.
\end{egg}
\begin{egg}
Let $X$ be any totally ordered set with at least two elements, and let 
$$\M{C}=\{I_x:x\in X\}\cup\{\emptyset\},$$ where $I_x=\{y\in X:y\leq x\}$ is an initial segment of $(X,<)$. For any $x,y \in X$ where $x\neq y$, without loss of generality $y<x$, we have 
$$\M{C}\sqcap \{x,y\}=\{\emptyset,\{y\},\{x,y\}\},$$ hence $\{y\}$ is shattered, however $\{x\}\notin\M{C}\sqcap \{x,y\}$ and so $\{x,y\}$ is not shattered. Therefore $\vc(X,\M{C})=1$.
\end{egg}
\begin{egg}
Let $X=\RR^2$ and 
$$\M{C}=\{[a,b]\times[c,d]:a,b,c,d\in\RR\}.$$ Clearly $\M{C}$ shatters $\{(-1,0),(0,1),(0,1),(0,-1)\}$. Now let 
$$A=\{(a_1,a_2),(b_1,b_2),(c_1,c_2),(d_1,d_2),(e_1,e_2)\}$$ be given. Without loss of generality $(a_1,a_2)$ is the leftmost point, $(b_1,b_2)$ is the highest point, $(c_1,c_2)$ is the rightmost point, and $(d_1,d_2)$ is the lowest point. Since
$$\{(a_1,a_2),(b_1,b_2),(c_1,c_2),(d_1,d_2)\}\s[a,b]\times[c,d]\in\M{C},$$ we have 
$$(e_1,e_2)\in[a_1,c_1]\times[d_2,b_2]\s[a,b]\times[c,d],$$ and so $$\{(a_1,a_2),(b_1,b_2),(c_1,c_2),(d_1,d_2)\}\notin\M{C}\sqcap A.$$ Therefore $\vc(X,\mathcal{C})$=4.
\end{egg}



Unless otherwise specified, from now on we will consider $(X,\mathcal{C})$ to be our concept space, and $d,k,m,n\in\NN$.

\begin{defn}[\cite{vapnik:264}]
The {\bfseries n'th shatter coefficients}\index{Shatter coefficients} of $\CC$ are defined to be $$s(\CC,n)=max\{|\CC\sqcap A|: A\subseteq X,\ |A|=n\}.$$
\end{defn}
Note that $\vc(\CC)=\sup\{n\in\NN:s(\CC,n)=2^n\}$.
\begin{notation}
Let $\binom{n}{\leq d}$ denote $$\binom{n}{\leq d}=\sum_{i=0}^{d}\binom{n}{i}.$$
\end{notation}
\begin{theo}[Sauer-Shelah Lemma \cite{MR0307902}]\index{Sauer-Shelah lemma}
Let $\vc(\CC)=d$. Then $\forall n\in\NN$ $$s(\CC,n)\leq\binom{n}{\leq d}\leq\left(\frac{en}{d}\right)^d.$$
\end{theo}

We can consider $\CC$ as ``function class"; a family of $\{0,1\}$ valued functions on $X$:
Let 
$$\FF_{\CC}=\{\chi_C: C \in \CC\}$$ where $\chi_C$ is the indicator function of $C$ on $X$. Similarly, if $\FF$ is a family of $\{0,1\}$ valued functions on $X$ we can get a concept class 
$$\CC_{\FF}=\{C\in2^X:\chi_C=f,\text{ for some } f\in\FF\}.$$
Defining shattering for a function class $\FF$ as: $A\subseteq X$ is shattered by $\FF$ if $\{f_{|_A}:f\in\FF\}=2^A$. We can see that $\FF$ shatters $A$ iff $\CC_{\FF}$ shatters $A$, and $\CC$ shatters $A$ iff $\FF_{\CC}$ shatters $A$ so the two notions are equivalent.\\ In the future we will consider concepts as functions, but will still use set relations and operations on concepts, which will have the obvious meaning; for instance $x\in C$ will be the same as $C(x)=1$, $C\s C'$ the same as support$(C)\s$ support$(C')$, $C\cap C'$ the same as $\min\{C,C'\}$, etc.

\section{Maximum and Maximal Classes}
The following definitions are due to \cite{welzl87rangespaces}.
\begin{defn}
Let $d\in\NN$. A concept class $\m{C}$ is {\bfseries d-maximum}\index{Maximum} if for every $A\s X$ finite, 
$$|\m{C}\sqcap A|=\binom{|A|}{\leq d}.$$
\end{defn}
\begin{defn}
A concept class $\m{C}$ is {\bfseries d-maximal}\index{Maximal} if $\vc(\m{C})=d$,\\and for any $D\in2^X\setminus\m{C}$ we have $\vc(\m{C}\cup\{D\})>d$.
\end{defn}

Note that if $\m{C}$ is $d$-maximum, then $\vc(\m{C})=d$ because for $A\s X$, if $|A|=d$ then 
$$|\m{C}\sqcap A|=\binom{d}{\leq d}=2^d=2^{|A|},$$
 so $A$ is shattered, and if $|A|>d$ then 
 $$|\m{C}\sqcap A|=\binom{|A|}{\leq d}<2^{|A|},$$ so $A$ is not shattered. 

As a consequence of Zorn's Lemma every concept class of VC dimension $d$ is contained in a $d$-maximal concept class.

Maximum does not necessarily imply maximal and vice versa. Also note that if $(X,\CC)$ is $d$-maximum, any subspace of $(X,\CC)$ is $d$-maximum as well, but this is not necessarily the case for $d$-maximal. 
\begin{egg} 
Let $X=\{1,2,3,4\}$, \\$\CC=\{\{1\},\{2\},\{3\},\{1,2\},\{1,3\},\{2,3\},\{1,4\},\{2,4\},\{3,4\},\{1,2,3\}\}$. It is easy to check $\m{C}$ is $2$-maximal but not $2$-maximum since 
$$|\m{C}|=10<11=\binom{4}{\leq 2}.$$
\end{egg}
\begin{egg}[\cite{Floyd95samplecompression}]
Let $X=\{1,2,3,4\}$, \\$\CC=\{\{1\},\{2\},\{1,2\},\{1,3\},\{2,3\},\{1,4\},\{2,4\},\{3,4\},\{1,2,3\},\{1,2,4\}\}$. It is easy to check $\m{C}$ is $2$-maximal but not $2$-maximum since 
$$|\m{C}|=10<11=\binom{4}{\leq 2}.$$
\end{egg}
\begin{egg}
Let $X=\RR$ and $\m{C}=\{(-\infty,a):a\in\mathbb{Q}\}$. For any $A=\{x_1,..., x_n\}\s X$ finite, without loss of generality with $x_1<...<x_n$,  we have that $$|\m{C}\sqcap A|=|\{\emptyset,\{x_1\},\{x_1,x_2\},...,\{x_1,..., x_n\}\}|=|A|+1=\binom{|A|}{\leq 1},$$ thus $\m{C}$ is $1$-maximum. However, $\m{C}$ is not $1$-maximal since $\vc(\m{C}\cup \{X\})=1$. Note that any concept space where $X$ is totally ordered with no minimal element, and where $\M{C}$ is the set of all initial segments, is $1$-maximum. This is also the case if $X$ has at least two elements, where $\M{C}$ is the set of all initial segments and the empty set.
\end{egg}

\begin{rmk}
If $(X,\m{C})$ is finite, then $d$-maximum implies $d$-maximal.\\
If $\CC$ is $d$-maximum, then any $A\in2^X\setminus\m{C}$ has $$|\m{C}\cup\{A\}|=|\m{C}|+1=\binom{|X|}{\leq d}+1>\binom{|X|}{\leq d}$$ hence by Sauer's Lemma $\vc(\m{C}\cup\{A\})>d$, and therefore $\CC$ is $d$-maximal.
\end{rmk}
\begin{lem}[\cite{welzl87rangespaces}]
Let $(X,\m{C})$ be finite with VC-dimension $d$. For $x\in X$, there are at most $\binom{|X|-1}{\leq d-1}$ sets $C\in\m{C}$ such that $x\in C$ and $C\setminus\{x\}\in\m{C}$.
\end{lem}
\begin{proof}
Let $x_0\in X$, 
$Y=X\setminus\{x_0\},$ and 
$$\M{C}'=\{C\in\M{C}:x_0\in C\text{ and }C\setminus\{x_0\}\in\m{C}\}.$$ Suppose 
$$|\M{C}'|>\binom{|X|-1}{\leq d-1}.$$ Then $$|\M{C}'\sqcap Y|=|\M{C}'|>\binom{|X|-1}{\leq d-1}=\binom{|Y|}{\leq d-1},$$ thus by Sauer's Lemma $\vc(Y,\m{C}'\sqcap Y)>d-1$. Let $\{x_1,...,x_d\}$ be $d$ points in $Y$ shattered by $\m{C}'\sqcap Y\s\M{C}$, and let $A=\{x_0,x_1,...,x_d\}$. Now by the definition of $\M{C}'$, for each $C\in\M{C}'\sqcap \{x_1,...,x_d\}=2^{\{x_1,...,x_d\}}$ there is $C_{x_0}\in\M{C}$ such that $C_{x_0}=C\cup\{x_0\}$, hence $$\m{C}\sqcap A\supset2^{\{x_1,...,x_d\}}\cup \{B\cup\{x_0\}:B\in2^{\{x_1,...,x_d\}}\}=2^A,$$ contradicting $\vc(\m{C})=d$.
\end{proof}
\begin{theo}[\cite{welzl87rangespaces}]
Let $(X,\m{C})$ be finite with VC-dimension $d$. The concept space $(X,\m{C})$ is $d$-maximum if and only if 
$$|\m{C}|=\binom{|X|}{\leq d}.$$
\end{theo}
\begin{proof}
If $(X,\m{C})$ is $d$-maximum then by the definition 
$$|\m{C}|=\binom{|X|}{\leq d}.$$
For the converse, we will use induction on $|X|=n\geq d$.\\
If $n=d$, then $\m{C}=2^X$ is maximum and 
$$|\m{C}|=2^d=\binom{d}{\leq d}.$$
Assume the statement of the theorem is true for all $(X,\m{C})$ where $|X|\leq n$, and let $(X,\m{C})$ have $|X|=n+1$. Let $x_0\in X$ and let $Y=X\setminus\{x_0\}$. By the induction hypothesis, it suffices to show that 
$$|\M{C}\sqcap Y|=\binom{n}{\leq d}.$$ By lemma 1.2.7,\\$\M{C}'=\{C\in\M{C}:x_0\in C\text{ and }C
\setminus\{x_0\}\in\m{C}\}$ has size at most $\binom{n}{\leq d-1}$. Define $$\pi:\M{C}\setminus \M{C}'\rightarrow \M{C}\sqcap Y\text{ by }\pi(C)=C\cap Y.$$ We will show $\pi$ is injective. Suppose there is $C_1\neq C_2$ in $\M{C}\setminus \M{C}'$ such that $$\pi(C_1)=C_1\cap Y=C_2\cap Y=\pi(C_2).$$ If $x_0\in C_1\cap C_2$, then $$C_1=(C_1\cap Y)\cup\{x_0\}=(C_2\cap Y)\cup\{x_0\}=C_2,$$ and if $x_0\notin C_1\cup C_2$, then $$C_1=C_1\cap Y=C_2\cap Y=C_2,$$ so without loss of generality $x_0\in C_1\setminus C_2$. We get that $$C_1\setminus\{x_0\}=C_1\cap Y=C_2\cap Y=C_2\in\M{C}$$ hence $C_1\in\M{C}'$, a contradiction, therefore $\pi$ is injective. Finally, $$|\M{C}\sqcap Y|\geq|\M{C}\setminus \M{C}'|=|\M{C}|-|\M{C}'|\geq\binom{n+1}{\leq d}-\binom{n}{\leq d-1}=\binom{n}{\leq d}.$$
\end{proof}

\section{Concepts as Relations}
~

In this section we will look at concept spaces defined as a relation on a pair of sets. This will allow us to characterize useful notions of embeddings for concept spaces as found in \cite{Ben-david98combinatorialvariability}. It will also allow us to define the dual concept space of a concept space.

We can define a concept class on a domain $X$ via a relation $R\s X\times Y$ for some set $Y$, by $\M{C}_R=\{C_y:y\in Y\}$ where $C_y=\{x\in X: (x,y)\in R\}$. Similarly given $(X,\m{C})$, the corresponding space in the form $(X,Y,R)$ is $(X,\M{C},\in)$. A subclass of $(X,Y,R)$ is  $(A,B,R_{|_{A\times B}})$ where $A\s X$, and $B\s Y$. This is convenient for defining the idea of a dual to a concept space as follows:
\begin{defn}
Given a concept space $(X,Y,R)$, the {\bfseries dual concept space}\index{Dual concept space} of $(X,Y,R)$, denoted $$(X,Y,R)^*,$$ is 
$$(Y,X,R^*),\text{ where }R^*=\{(y,x):(x,y)\in R\}.$$ The dual concept space of a space represented as $(X,\m{C})$, can be thought of as 
$$(\m{C},\{\{C\in\M{C}:x\in C\}:x\in X\}).$$
\end{defn}

\begin{defn}[\cite{Ben-david98combinatorialvariability}]
Let $(X,Y,R)$, $(X',Y',R')$ be concept spaces. An {\bfseries embedding}\index{Embeddable}\index{Generalized embeddable} from $(X,Y,R)$ to $(X',Y',R')$ is a function $\pi:X\times Y\rightarrow X'\times Y'$ such that for every $$(x,y)\in X\times Y,\ (x,y)\in R\text{ iff }\pi((x,y))\in R'.$$\\A {\bfseries generalized embedding} from $(X,Y,R)$ to $(X',Y',R')$ is a function $\tau\in2^X$ and a function $\pi:X\times Y\rightarrow X'\times Y'$ such that for every $(x,y)\in X\times Y$,
\begin{align*} 
&\text{ if } \tau(x)=0 \text{ then }(x,y)\in R \text{ iff }\pi((x,y))\in R',\\
&\text{ if }\tau(x)=1\text{ then }(x,y)\in R \text{ iff }\pi((x,y))\notin R'.
\end{align*}

$(X,Y,R)$ is {\bfseries weakly (generalized) embeddable}\index{Weakly embeddable}\index{Weakly generalized embeddable} in $(X',Y',R')$ if every finite subclass $(A,B,R_{|_{A\times B}})$ of $(X,Y,R)$ is (generalized) embeddable in $(X',Y',R')$.
\end{defn}

The above notions partially order any set of concept spaces; if there exists an embedding or generalized embedding from $(X,Y,R)$ to $(X',Y',R')$, we will denote that $$(X,Y,R)\preceq_{emb}(X',Y',R')$$ or $$(X,Y,R)\preceq_{gemb}(X',Y',R')$$ respectively.\\If $(X,Y,R)$ is weakly embeddable in $(X',Y',R')$, or weakly generalized embeddable in $(X',Y',R')$, we will denote that $$(X,Y,R)\preceq^w_{emb}(X',Y',R')$$ or $$(X,Y,R)\preceq^w_{gemb}(X',Y',R')$$ respectively.

\begin{defn}
Let us say that $(X,Y,R)$ and $(X',Y',R')$ are {\bfseries bi-embeddable}\index{Bi-embeddable} if $(X,Y,R)\preceq_{emb}(X',Y',R')$ and $(X',Y',R')\preceq_{emb}(X,Y,R)$.
\end{defn}

A concept space $(X,Y,R)$ may have some redundant points in $X\times Y$ as far as $R$ is concerned, but we can reduce it to its essential information by setting:
\begin{align*}
&x\sim x' \text{ in }X \text{ iff }\forall y\in Y,\ (x,y)\in R \iff (x',y)\in R,\\
&y\sim y' \text{ in }Y \text{ iff }\forall x\in X,\ (x,y)\in R \iff (x,y')\in R.
\end{align*}
$$R_{\sim}=\{([x]_{\sim},[y]_{\sim})\in X/\mathord\sim\times Y/\mathord\sim: (x,y)\in R\}$$ separates the points of $X/\mathord\sim\times Y/\mathord\sim$ and $(X/\mathord\sim,Y/\mathord\sim,R_{\sim})$ is bi-embeddable to $(X,Y,R)$ via the quotient map for $(X,Y,R)\preceq_{emb}(X/\mathord\sim,Y/\mathord\sim,R_{\sim})$, and mapping each equivalence class to its (choose any) representative for \\ $(X/\mathord\sim,Y/\mathord\sim,R_{\sim})\preceq_{emb}(X,Y,R)$.

\begin{rmk}
~
\begin{description}
\item (1) $(X,Y,R)\preceq_{emb}(X',Y',R')\Rightarrow(X,Y,R)\preceq_{gemb}(X',Y',R')\Rightarrow(X,Y,R)\preceq^w_{gemb}(X',Y',R')$.
\item (2) $(X,Y,R)\preceq^w_{emb}(X',Y',R')\Rightarrow(X,Y,R)\preceq^w_{gemb}(X',Y',R')$.
\end{description}
\end{rmk}
\begin{notation}
In the proof of the next proposition and throughout the further text we use the notation $\triangle$ for symmetric difference of a set; i.e. $V\triangle W:=(V\setminus W)\cup(W\setminus V)$
\end{notation}
\begin{prop}[\cite{Ben-david98combinatorialvariability}]
If $(X,Y,R)\preceq^w_{gemb}(X',Y',R')$ then $\vc(X,Y,R)\leq\vc(X',Y',R')$.
\end{prop}
\begin{proof}
Let $A$ be a finite subset of $X$ that is shattered, let $$B=\{b_D\in Y:D\s A,\ C_{b_D}\cap A=D\},$$ and let $\tau$, $\pi=(\pi_1,\pi_2)$ be the generalized embedding from $(A,B,R)$ into $(X',Y',R')$. $\pi_2$ is injective because for $b_1,b_2\in B,\ b_1\neq b_2$, there exists $x\in C_{b_1}\setminus C_{b_2}\cup C_{b_2}\setminus C_{b_1}$. Without loss of generality $x\in C_{b_1}\setminus C_{b_2}$. We have:
\begin{align*}
&\text{ if }\tau(x)=0,\text{ then } x\in C_{b_1}\text{ implies }\pi_1(x)\in C_{\pi_2(b_1)}\text{ and }x\notin C_{b_2}\text{ implies }\pi_1(x)\notin C_{\pi_2(b_2)};\\
&\text{ if }\tau(x)=1,\text{ then } x\in C_{b_1}\text{ implies }\pi_1(x)\notin C_{\pi_2(b_1)}\text{ and }x\notin C_{b_2}\text{ implies }\pi_1(x)\in C_{\pi_2(b_2)}.
\end{align*}
In either case $\pi_1(x)\in C_{\pi_2(b_1)}\triangle C_{\pi_2(b_2)}$ and so $\pi_2(b_2)\neq\pi_2(b_1)$. This also shows that $C_b\mapsto C_{\pi_2(b)}\cap\pi_1(A)$ is injective, hence $$2^{|A|}\geq|\{C_{\pi_2(b)}\cap\pi_1(A):b\in B\}|\geq|\{C_{b}:b\in B\}|=2^{|A|}$$ and therefore $\pi_1(A)$ is shattered in $(X',Y',R')$.
\end{proof}
\begin{theo}[\cite{Laskowski92vapnik-chervonenkisclasses}]
For any class $(X,Y,R)$:
$$\log_2(\vc(X,Y,R))-1<\vc((X,Y,R)^*)<2^{\vc(X,Y,R)+1}.$$
\end{theo}
\begin{proof}
Since $((X,Y,R)^*)^*=(X,Y,R)$, it suffices to show the first inequality. Let $A$ be a set of cardinality $\lfloor\log_2(\vc(X,Y,R))\rfloor$. One has $(A,2^A,\in)\preceq_{emb}(2^A,2^{2^A},\in)^*$ via $\pi(x,y)=(\{B\s A: x\in B\},y)$. Noting that $(2^A,2^{2^A},\in)$ is embeddable in any class of the same or greater VC-dimension, $(2^A,2^{2^A},\in)\preceq_{emb}(X,Y,R)$, and thus $(2^A,2^{2^A},\in)^*\preceq_{emb}(X,Y,R)^*$. Therefore $(A,2^A,\in)\preceq_{emb}(X,Y,R)^*$ and so $\lfloor\log_2(\vc(X,Y,R))\rfloor\leq\vc((X,Y,R)^*)$.
\end{proof}
\begin{cor}
$\vc(X,Y,R)<\infty$ if and only if $\vc((X,Y,R)^*)<\infty$.
\end{cor}



\cleardoublepage

\chapter[Sample Compression Schemes]{Sample Compression Schemes}

\section{Introduction of Sample Compression Schemes}
~~~

Sample compression schemes, introduced by Littlestone and Warmuth (\cite{Littlestone86relatingdata}), are naturally arising algorithms which learn concepts by saving finite samples of concepts to subsets of size at most $d$.

The following notations will be used in the definitions of sample compression schemes, and throughout the text.
\begin{notation}
For $d\in\NN\cup\{\infty\}$ let $$[X]^{< d}=\{A\s X:|A|< d\},$$ let 
$$\CC_{|_A}=\{C_{|_A} :C\in\CC\},$$ where $A\s X$ and $C_{|_A}$ is the function $C$ restricted to the domain $A$, and let
$$\CC_{|_{[X]^{< d}}}=\{C_{|_A} :C\in\CC,\ A\subseteq X,\ |A|<d\}.$$
We can similarly define 
$$[X]^{\leq d},\ \CC_{|_{[X]^{\leq d}}},\ [X]^{= d},\text{ and }\CC_{|_{[X]^{= d}}}.$$
\end{notation}
\begin{notation}
For two functions $f$,$g$ with $\dom(g)\s \dom(f)$, let $$g\sqsubseteq f$$ be the notation for $f$ extending $g$.
\end{notation}
\begin{defn}
For $d\in\NN$, an {\bfseries unlabelled sample compression scheme of size d}\index{Unlabelled sample compression scheme} on $(X,\mathcal{C})$ is a function 
$$\m{H}:[X]^{\leq d}\rightarrow 2^X$$
with the property that 
$$\forall f\in\CC_{|_{[X]^{< \infty}}},\ \exists \0\in [\dom(f)]^{\leq d},\text{ such that }f\sqsubseteq \m{H}(\0).$$
A {\bfseries labelled sample compression scheme of size d}\index{Labelled sample compression scheme} on $(X,\mathcal{C})$ is a function 
$$\m{H}:\CC_{|_{[X]^{\leq d}}}\rightarrow 2^X$$
with the property that 
$$\forall f\in\CC_{|_{[X]^{< \infty}}},\ \exists g\in \CC_{|_{[X]^{\leq d}}},\text{ such that }g\sqsubseteq f\sqsubseteq \m{H}(g).$$
We will call the range of a sample compression scheme the {\bfseries hypothesis class }\index{Hypothesis class} and denote it by $\mathfrak{H}$.
\end{defn}

\begin{egg}
Let $X$ be any totally ordered set, and let $\M{C}=\{I_x:x\in X\}$ be the set of all initial segments of $X$. Defining
\begin{align*}
&\m{H}:\{x\}\mapsto I_x,\ \emptyset\mapsto\emptyset,\\
\text{ and }&\m{H'}:\{x\}\mapsto I_x\setminus\{x\},\ \emptyset\mapsto X,
\end{align*}
we will show $\HH$ and $\HH'$ are unlabelled sample compression schemes of size $1$ on $(X,\CC)$. Given a sample $f\in\CC_{|_{[X]^{< \infty}}}$, if $f=0$ on its domain then $\emptyset\in [\dom(f)]^{\leq 1}$ and $f \sqsubseteq \m{H}(\emptyset)=\emptyset$. Otherwise $x_f=\max\{x:f(x)=1\}$ exists, and so 
$$\{x_f\}\in [\dom(f)]^{\leq 1},\ f \sqsubseteq \m{H}(\{x_f\})=I_{x_f}.$$ Thus $\HH$ is a sample compression scheme of size $1$ on $(X,\CC)$.\\
Similarly for $\M{H}'$, if $f=1$ on its domain then $\emptyset\in [\dom(f)]^{\leq 1}$ and $f \sqsubseteq \m{H}(\emptyset)=X$. Otherwise $x_f=\min\{x:f(x)=0\}$ exists, and so 
$$\{x_f\}\in [\dom(f)]^{\leq 1},\ f \sqsubseteq \m{H}(\{x_f\})=I_{x_f}\setminus\{x_f\}.$$ Therefore $\HH'$ is also a sample compression scheme of size $1$ on $(X,\CC)$.

\end{egg}

\begin{prop}
If $(X,\mathcal{C})$ has an unlabelled compression scheme of size $d$, then $(X,\mathcal{C})$ has a labelled compression scheme of size $d$.
\end{prop}
\begin{proof}
Let $(X,\mathcal{C})$ have an unlabelled compression scheme $\M{H}$ of size $d$. For every $f\in\CC_{|_{[X]^{< \infty}}}$ there is $\0_f\in [\dom(f)]^{\leq d}$ such that $f\sqsubseteq \m{H}(\0_f)$, and so any function $\M{H}':\CC_{|_{[X]^{\leq d}}}\rightarrow 2^X$ where $\M{H}'(f_{|_{\0_f}})=\m{H}(\0_f)$ will be a labelled compression scheme of size $d$.
\end{proof}

From now on we will only be dealing with unlabelled sample compression schemes unless otherwise mentioned.

\begin{prop}[\cite{Ben-david98combinatorialvariability}]
If $(X,Y,R)\preceq^w_{gemb}(X',Y',R')$ and $(X',Y',R')$ has a (labelled or unlabelled) sample compression scheme of size $d$, then $(X,Y,R)$ also has a sample compression scheme of size $d$ and of the same type.
\end{prop}
\begin{cor}
If $(X,\mathcal{C})$ has a sample compression scheme of size $d$, then every subspace has a sample compression scheme of size $d$.
\end{cor}

\section{Compactness Theorem}

\begin{theo}[Compactness Theorem, Ben-David and Litman \cite{Ben-david98combinatorialvariability}]\index{Compactness theorem}
A concept space $(X,\CC)$ has a sample compression scheme of size $d$ if and only if every finite subspace of $(X,\CC)$ has a sample compression scheme of size $d$.
\end{theo}

The compactness theorem is true for both types of sample compression schemes and similarly for all forms of extended sample compression schemes given in a following section. We will provide the proof of the theorem for unlabelled sample compression schemes. The proof we provide is simpler and more direct than the proof in \cite{Ben-david98combinatorialvariability} which is based on the Compactness Theorem of Predicate Logic. We use an approach with ultralimits, normally used in Analysis. (For preliminary information on filters and ultrafilters, see appendix A.2)

\begin{proof}
\emph{Necessity}: By corollary 2.1.7 if $(X,\CC)$ has a sample compression scheme of size $d$ every (finite) 
 subspace of $(X,\CC)$ has a sample compression scheme of size $d$.
\\
\emph{Sufficiency}: For all $ A\in[X]^{<\infty}$ denote the sample compression scheme of size $d$ for $(A,\CC\sqcap A)$ as $\m{H}_{A}$. Let $\mathfrak{U}$ be an ultrafilter on $[X]^{<\infty}$ containing the filter base 
$$\{\{B\in[X]^{<\infty}:F\s B\}:F\in[X]^{<\infty}\}.$$
Define $\m{H}:[X]^{\leq d}\rightarrow2^{X}$ as 
$$ \m{H}(\0)(x)=1\iff\{B\in[X]^{<\infty}:\0\cup\{x\}\s B,\ \m{H}_{B}(\0)(x)=1 \}\in\mathfrak{U}.$$
Note for given $\0\in [X]^{\leq d},\ x\in X$, $\m{H}(\0)(x)$ is defined as the ultralimit of the net of zeros and ones $\{\m{H}_{A}(\0)(x)\}_{ \0\s A\in[X]^{<\infty}}$ along $\mathfrak{U}$.

We will show $\m{H}$ is a sample compression scheme of size $d$ on $(X,\CC)$. Let $f\in\CC_{|_{[X]^{< \infty}}}$, and denote $\dom(f)=D$. Note that 
\begin{equation}
\forall B\in[X]^{<\infty},\ D\s B,\text{ we have }f\in(\CC\sqcap B)_{|_{[X]^{< \infty}}},\text{ and so }\exists \0_{B}\in[D]^{\leq d}\text{ such that }f\sqsubseteq \m{H}_B(\0_{B}).
\tag{1}\end{equation}
We have that $[D]^{\leq d}$ is finite so let $\{\0_1,...,\0_m\}=[D]^{\leq d}$. For $i\in\{1,...,m\}$ letting $$\m{S}_i=\{B\in[X]^{<\infty}:D\s B,\ f\sqsubseteq \m{H}_{B}(\0_i) \},$$ by (1) we see that $$\bigcup_{i=1}^{m}\m{S}_i=\{B\in[X]^{<\infty}:D\s B\}\in\mathfrak{U}$$ thus, by a property of ultrafilters, $\exists i_0$ such that $\m{S}_{i_0}\in\mathfrak{U}$. Let $x\in D$ and let 
$$\m{S}_{i_0}^{t}=\{B\in[X]^{<\infty}:D\s B,\ \m{H}_{B}(\0_{i_0})(x)=t \}\text{ (where }t\in\{0,1\}).$$
We have
\begin{alignat*} {2}
&f(x)=1\ &\Rightarrow&\ \forall B\in \m{S}_{i_0},\ \m{H}_{B}(\0_{i_0})(x)=1\\ 
&{}&\Rightarrow&\ \m{S}_{i_0}\s\m{S}_{i_0}^{1}\s\{B\in[X]^{<\infty}:\0_{i_0}\s B,\ \m{H}_{B}(\0_{i_0})(x)=1 \}\in\mathfrak{U}\\ 
&{}&\Rightarrow&\ \m{H}(\0_{i_0})(x)=1;\\
&f(x)=0\ &\Rightarrow&\ \forall B\in \m{S}_{i_0},\ \m{H}_{B}(\0_{i_0})(x)=0\\ 
&{}&\Rightarrow&\ \m{S}_{i_0}\s\m{S}_{i_0}^{0}\s\{B\in[X]^{<\infty}:\0_{i_0}\s B,\ \m{H}_{B}(\0_{i_0})(x)=0 \}\in\mathfrak{U}\\
&{}&\Rightarrow&\ \{B\in[X]^{<\infty}:\0_{i_0}\s B,\ \m{H}_{B}(\0_{i_0})(x)=0 \}^c\notin\mathfrak{U}\\
&{}&\Rightarrow&\ \{B\in[X]^{<\infty}:\0_{i_0}\s B,\ \m{H}_{B}(\0_{i_0})(x)=1 \}\notin\mathfrak{U}\\&{}&\Rightarrow&\ \m{H}_{B}(\0_{i_0})(x)=0.
\end{alignat*}
Therefore $f\sqsubseteq \m{H}(\0_{i_0})$ and so $\m{H}$ is a sample compression scheme of size $d$ on $(X,\CC)$.
\end{proof}

We would like to point out that even though the above proof is essentially the same as the original proof in \cite{Ben-david98combinatorialvariability}, it is reformulated using ultralimits, as usually done in analysis, and does not use logic. As such, it may be easier to understand.

A point of concern with the compactness theorem is that the compression scheme resulting from the  finite domains need not have measurable hypothesis spaces, and we construct an example to prove this point in chapter 4. This problem is the main focus of chapter 4.

\section{Sample Compression Schemes and VC Dimension}

The following remark is a simple observation.
\begin{prop}
If $(X,\mathcal{C})$ has an unlabelled sample compression scheme of size $d$, then
$\vc(\M{C})\leq d$.
\end{prop}
\begin{proof}
Suppose $\vc(\CC)>d$ and let $A\s X$ be a set of size $d+1$ which is shattered. We have that $|\CC\sqcap A|=2^{d+1}$, but there are 
$$\binom{d+1}{\leq d}<\binom{d+1}{\leq d+1}=2^{d+1}$$ 
subsets of $A$ of size at most $d$. Therefore $(A,\mathcal{C}\sqcap A)$ cannot have a sample compression scheme of size $d$, which is a contradiction.
\end{proof}
Note that the proof only applies to unlabelled sample compression schemes and the same is not true for labelled sample sample compression schemes. However if a labelled sample compression scheme of size $d$ exists on $(X,\mathcal{C})$ then $\vc(\M{C})\leq 5d$ \cite{Floyd95samplecompression}.

It is a major open question posed by Littlestone and Warmuth, in paper \cite{Littlestone86relatingdata}, whether or not a concept space $(X,\mathcal{C})$ has an (unlabelled or labelled) sample compression scheme of size O$(\vc(\m{C}))$. In the case of $d$-maximum spaces, a sample compression scheme of size $d$ exists.

\begin{theo}[\cite{Kuzmin06unlabeledcompression}]
If $(X,\mathcal{C})$ is $d$-maximum, then it has an unlabelled compression scheme of size $d$.
\end{theo}

This thesis' motivation is to try and generalize this result to obtain measurable hypotheses. We succeed in chapter 4 under some additional assumptions.

Using remark 1.2.6 we have the following corollary.
\begin{cor}
If $(X,\mathcal{C})$ is finite with VC-dimension $d$ and 
$$|\mathcal{C}|=\binom{|X|}{\leq d},$$ then it has an unlabelled compression scheme of size $d$.
\end{cor}

\section{Extended Sample Compression Schemes}
~~~

In this section we will introduce a new variant of compression scheme called copy sample compression schemes. All other discussed compression schemes are special cases of copy sample compression schemes. The initial motivation for copy sample compression schemes was the ability to collect sample compression schemes of varying sizes, and for different concept classes, into one function; a copy sample compression scheme can be thought of as an algorithm which checks sample compression schemes of varying sizes, and for different concept classes, and picks a hypothesis for the concepts in any of these different concept classes. This is formalized in proposition 2.4.7. Copy sample compression schemes also add other flexibilities, and may in some instances have better bounds for sample complexity than regular sample compression schemes.

\begin{defn}[\cite{Ben-david98combinatorialvariability}]
Let $\mathbf{b}$ be a symbol not in $X$.
\\An {\bfseries array sample compression scheme of size k}\index{Array sample compression scheme} for $(X,\CC)$ is a function 
$$\m{H}:(X\cup\{\bold{b}\})^d\rightarrow 2^X$$ 
with the property that 
$$\forall f\in\CC_{|_{[X]^{< \infty}}},\ \exists \0\in (X\cup\{\bold{b}\})^d,\text{ such that }\range(\0)\s\dom(f)\cup\{\bold{b}\}\text{ and }f\sqsubseteq \m{H}(\0)$$ 
(where, for a sequence $\0= (a_1,. ..,a_k)$, $\range(\0)$ is the set $\{a_1,. ..,a_k\}$).
\end{defn}
\begin{defn}[\cite{Littlestone86relatingdata}]
An {\bfseries extended sample compression scheme of size k using b bits}\index{Extended sample compression scheme} for $(X,\CC)$ is a function 
$$\m{H}:\bigcup_{i=0}^k([X]^{= i}\times2^b)\rightarrow 2^X$$
with the property that 
$$\forall f\in\CC_{|_{[X]^{< \infty}}},\ \exists \0\in [\dom(f)]^{\leq k}\text{ and }\tau\in2^b,\text{ such that }f\sqsubseteq \m{H}(\0\times\tau).$$
\end{defn}

The preceding definitions are special cases of the following new variant of an extended sample compression scheme.

\begin{defn}
Let $k\in\NN$, and let $\{n_i\}_{i=0}^{k}$ be a finite sequence in $\NN$.\\
A {\bfseries $\bold{\{n_i\}_{i=0}^{k}}$-copy unlabelled sample compression scheme of size k}\index{Copy sample compression scheme} on $(X,\mathcal{C})$ is a function 
$$\m{H}:\bigcup_{i\in\{j\in\NN:0\leq j\leq k,\ n_j\neq 0\}}([X]^{= i}\times\{1,...,n_i\})\rightarrow 2^X$$ with the property that 
$$\forall f\in\CC_{|_{[X]^{< \infty}}},\ \exists \0\in [\dom(f)]^{\leq k}\text{ and }i\in\{1,...,n_{|\0|}\},\text{ such that }f\sqsubseteq \m{H}(\0\times i).$$ 
If $\{n_i\}_{i=0}^{k}$ is just a constant sequence $\{n\}$ for some $n\in\NN$, we will call it an {\bfseries n-copy unlabelled sample compression scheme of size k}. We can define $\{n_i\}_{i=0}^{k}$-copy labelled sample compression schemes of size $k$ similarly.
\end{defn}

Note that a sample compression scheme of size $d$ defines a $1$-copy sample compression scheme of size $d$, and a $1$-copy sample compression scheme of size $d$ defines a sample compression scheme of size $d$. A compactness theorem can also be proven for copy sample compression schemes (and the other versions of extended sample compression schemes), namely, $(X,\CC)$ has a $\{n_i\}_{i=0}^{k}$-copy sample compression scheme of size $k$ if and only if every finite subspace $(Y,\CC\sqcap Y)$ has a $\{n_i\}_{i=0}^{k}$-copy sample compression scheme of size $k$.


\begin{prop}
Let $|X|=m$, let $(X,\CC)$ have a sample compression scheme of size $d$, and let $k\leq d$. Whenever 
$$n\binom{m}{\leq k}\geq\binom{m}{\leq d},$$ $(X,\CC)$ has an $n$-copy sample compression scheme of size $k$.
\end{prop}

Having an $n$-copy sample compression scheme of size $k$ for $(X,\CC)$ does not imply there is a sample compression scheme of size $d$ when 
$$n\binom{m}{\leq k}\leq\binom{m}{\leq d}.$$

\begin{egg}
Let $X=\{1,2,3,4,5\}$, $\CC=2^{\{1,2\}}$. Enumerate $\CC=\{C_l:l\in\{1,2,3,4\}\}$ in any way and define $\HH:[X]^{=0}\times\{1,2,3,4\}\rightarrow 2^X$ by $\HH(\emptyset\times l)=C_l$, $l\in\{1,2,3,4\}$. $\HH$ is a $4$-copy sample compression scheme of size $0$, but since $\vc(X,\CC)=2$, $(X,\CC)$ has no sample compression scheme of size $1$ although 
$$4\binom{5}{0}=4<5=\binom{5}{1}$$
\end{egg}
\begin{egg}
Let $X=\{1,2,3,4\}$, \\$\CC=\{\{1\},\{2\},\{3\},\{1,2\},\{1,3\},\{2,3\},\{1,4\},\{2,4\},\{3,4\},\{1,2,3\}\}$. Recall that $\m{C}$ is $2$-maximal but not $2$-maximum since 
$$|\m{C}|=10<11=\binom{4}{\leq 2}.$$ We can define a $2$-copy sample compression scheme of size $1$ by\\ 
\begin{align*}
\HH(\emptyset\times 1)=\{1,2\}&,\ \HH(\emptyset\times 2)=\{3,4\},\\
\HH(\{1\}\times 1)=\{3\}&,\ \HH(\{1\}\times 2)=\{1,3\},\\
\HH(\{2\}\times 1)=\{1\}&,\ \HH(\{2\}\times 2)=\{2,4\},\\
\HH(\{3\}\times 1)=\{2\}&,\ \HH(\{3\}\times 2)=\{1,2,3\},\\
\HH(\{4\}\times 1)=\{2,3\}&,\ \HH(\{4\}\times 2)=\{1,4\}.
\end{align*}
Note that 
$$2\binom{4}{\leq 1}=10,$$ however $(X,\CC)$ has no sample compression scheme of size less than 2.
\end{egg}

Copy sample compression schemes for a concept space $(X,\CC)$ can be defined from multiple sample compression schemes for concept classes which cover $\CC$. This can allow us to split up $\CC$ into spaces which are known to have sample compression schemes, and then form a copy sample compression scheme.
\begin{prop}
Let $|X|=m$, and $\CC\s\bigcup_{j=1}^{n}\CC_j$ where each $(X,\CC_j)$ has a sample compression scheme $\HH_j$ of size $d_j$. Define 
$$k=\max(\{d_j\}_{j=1}^n),$$ and 
$$n_i=|\{j:d_j\geq i\}|$$ for $0\leq i\leq k$. Then $(X,\CC_j)$ has a $\{n_i\}_{i=0}^{k}$-copy sample compression scheme of size $k$.
\end{prop}
\begin{proof}
Without loss of generality assume $j< l$ implies $d_j\geq d_l$. We will show that $\HH$ defined by $\HH(\0\times l)=\HH_l(\0)$ for $l\in\{1,...,n_{|\0|}\}$ is a $\{n_i\}_{i=0}^{k}$-copy sample compression scheme of size $k$ for $(X,\CC)$. Note that $\HH_l$ is defined at $\0$ because 
$$|\{1,...,n_{|\0|}\}|=n_{|\0|}=|\{j:d_j\geq |\0|\}|$$ implies 
$$d_1\geq|\0|,\ d_2\geq|\0|,...,d_{n_{|\0|}}\geq|\0|$$ so in particular $d_l\geq |\0|$.\\
Let $f\in\CC_{|_{[X]^{< \infty}}}$ and let $C_f\in\CC\s\bigcup_{j=1}^{n}\CC_j$ be such that $C_f$ extends $f$ to $X$. Since $C_f\in\CC_l$ for some $1\leq l\leq n$, there exists 
$$\0\in[\dom(f)]^{\leq d_l}\s[\dom(f)]^{\leq k}$$ such that 
$$f\sqsubseteq\HH_l(\0)=\HH(\0\times l).$$
\end{proof}
\begin{cor}
If there exists a family $\{\CC_j\}_{j=1}^n$ of concept classes such that $\CC\s\bigcup_{j=1}^n\CC_j$, and for each $\CC_j$ there is a $d$-maximum $\CC'_j$ containing $\CC_j$, then there is an $n$-copy sample compression scheme of size $d$ for $(X,\CC)$.
\end{cor}

\cleardoublepage

\chapter[Learnability]{Learnability}

\section{Learnability}
~~

The PAC, or ``Probably Approximately Correct", model for learning was introduced by Valiant \cite{Valiant:1984:TL:800057.808710}. In this chapter we introduce PAC learnability, and investigate sample complexity for spaces with finite VC dimension, and for spaces with sample compression schemes or copy sample compression schemes. 

For this section, let $(X,\A)$ be a measurable space with a concept class $\CC\s\A$ and a family of probability measures $\m{P}$ on $(X,\A)$. We will also reference measurability conditions (M1),(M2),...,(M5) defined and discussed in appendix B.1.

\begin{defn}[\cite{Valiant:1984:TL:800057.808710}]
A {\bfseries learning rule}\index{Learning rule} for $(X,\CC)$ is a function 
$$\m{L}:\bigcup_{n=1}^\infty(X^n\times2^n)\rightarrow\A$$
which satisfies the following measurability condition we will label ``(M1)": for every $C\in\CC$, every $n\geq1$ and every $P\in\m{P}$, the function 
$$A\mapsto P(\LL(A,C_{|_A})\triangle C)$$
from $(X,\A)^n$ to $\RR$ is measurable.\\
We will call $\mathfrak{H}=\range(\LL)\s\A$ the {\bfseries hypothesis space}.\\
A learning rule is {\bfseries consistent} with $\CC$ if for every $C\in\CC$ and $A\in X^n$, $$\LL(A,C_{|_A})_{|_A}=C_{|_A}.$$

\end{defn}

The domain $\bigcup_{n=1}^\infty(X^n\times2^n)$ for learning rules represents all finite labelled samples where $(A,\tau)$, for $A=(a_1,...,a_n) \in X^n$ and $\tau=(l_1,...,l_n)\in2^n$, represents the function $\{(a_1,l_1),...,(a_n,l_n)\}$, and $(A,C_{|_A})$ in the definition is shorthand for $(A,(\chi_{C\cap A}(a_1),...,\chi_{C\cap A}(a_n)))$.

Any sample compression scheme $\HH$ of size $d$ on $(X,\CC)$ defines a
(in general, non-unique) function $\LL_\HH$, with $\LL_\HH$ being a consistent learning rule for $(X,\CC)$ if it satisfies (M1):\\
For each $A\in [X]^{<\infty},\ C\in\CC\sqcap A$ pick $\0_{C,A}\in[A]^{\leq d}$ such that $C=\HH(\0_{C,A})_{|_A}$. Define 
$$\LL_\HH(A,\tau)=\begin{cases} \HH(\0_{C,A}), & \mbox{if } \tau=C \mbox{ for some } C\in\CC\sqcap A\\ \emptyset, & \mbox{otherwise} \end{cases}.$$
Similarly a $\{n_i\}_{i=0}^{k}$-copy sample compression scheme $\HH'$ of size $k$ defines a (in general, non-unique) function $\LL_{\HH'}$, with $\LL_{\HH'}$ being a a consistent learning rule for $(X,\CC)$ if it satisfies (M1):\\
For each $A\in [X]^{<\infty},\ C\in\CC\sqcap A$ pick $\0_{C,A}\in[A]^{\leq d},\ l\in\{1,...,n_{|\0|}\}$ such that $C=\HH'(\0_{C,A}\times l)_{|_A}$. Define 
$$\LL_{\HH'}(A,\tau)=\begin{cases} \HH'(\0_{C,A}\times l), & \mbox{if } \tau=C \mbox{ for some } C\in\CC\sqcap A\\ \emptyset, & \mbox{otherwise} \end{cases}.$$


\begin{defn}[\cite{Valiant:1984:TL:800057.808710}]
A learning rule $\LL$ for a space $(X,\CC)$ is {\bfseries probably approximately correct (PAC)}\index{Probably approximately correct} under $\m{P}$ if for every $\e>0$
$$\lim_{n\rightarrow\infty}\sup_{P\in\m{P}}\sup_{C\in\CC}P^n(\{A\in X^n:P(\LL(A,C_{|_A})\triangle C)>\e\})=0.$$ 
We will say a concept space $(X,\CC)$ is {\bfseries probably approximately correct (PAC) learnable} (under $\m{P}$) if there exists a PAC learning rule (under $\m{P}$) for $(X,\CC)$. For a given PAC learning rule $\LL$ and given $0<\e\leq1,\ 0<\delta\leq1$ we will define the {\bfseries sample complexity}\index{Sample complexity} of $\LL$, $$m_\LL(\e,\delta),$$
to be the least integer such that for all $n\geq m_\LL(\e,\delta)$, 
$$\sup_{P\in\m{P}}\sup_{C\in\CC}P^n(\{A\in X^n:P(\LL(A,C_{|_A})\triangle C)>\e\})<\delta.$$
We call $\e$ the {\bfseries accuracy}\index{Accuracy} and $\delta$ the {\bfseries risk}\index{Risk}.
\end{defn}

Under the measurability condition that $\CC $ is well behaved ((M4) in appendix B.1), 
the following is a result due to Blumer et al in \cite{Blumer:1989:LVD:76359.76371}:
\begin{theo}[\cite{Blumer:1989:LVD:76359.76371}]
The following conditions are equivalent:
\begin{description}
\item(1) $\vc(X,\CC)<\infty$.
\item(2) $(X,\CC)$ is PAC learnable.
\item(3) Every consistent learning rule $\LL:\bigcup_{n=1}^\infty(X^n\times2^n)\rightarrow\CC$ is PAC  for $(X,\CC)$ with
$$m_\LL(\e,\delta)\leq\max\left(\frac{4}{\e}\log_2(\frac{2}{\delta}),\frac{8d}{\e}\log_2(\frac{13}{\e})\right)$$
for $d=\vc(X,\CC)$.
\end{description}
\end{theo}
Given a consistent learning rule with $\hh\s\CC$ as in (3) in the above theorem, if we assume a measurability condition (M3) (appendix B.1) which is stronger than the condition of being well behaved, 
 it is shown by Shawe-Taylor et al. in \cite{Shawe-taylor93boundingsample} that we can improve the bounds: Every consistent learning rule with $\hh\s\CC$, and satisfying (M3), has sample complexity
$$m_\LL(\e,\delta)\leq\frac{1}{1-\beta}\left(\frac{1}{\e}\ln(\frac{2}{\delta})+\frac{2d\ln2}{\e}+\frac{d}{\e}\ln(\frac{1}{\e \beta^2})\right)$$
for any $0<\beta<1$, where $d$ is the VC dimension of $(X,\CC)$. 

\section{Sample Complexity of Compression Schemes}

\begin{theo}[\cite{Littlestone86relatingdata}]
Let $P$ be any probability measure on a measurable space $(X,\mathfrak{A})$, $C$ a concept in $\m{C}\s \mathfrak {A}$, and $\m{H}$ any function from $[X]^{\leq d}$ to $2^X$, satisfying measurability condition (M5). Then the probability that $A\s X$, $|A|=m\geq d$, contains a subset $\0$ of size at most $d$ such that $P(\M{H}(\0)\triangle C)>\varepsilon >0$ and $\M{H}(\0)_{|_A}=C_{|_A}$, is at most 
$$\sum_{n=0}^d\binom{m}{n}(1-\varepsilon)^{m-n}.$$
\end{theo}
\begin{proof}
Let $C\in\CC$ and $\varepsilon$ be given.
First we consider the probability that a set of size $m$ has a subset of size exactly $n\leq d$ with the property $P(\M{H}(\0)\triangle C)>\varepsilon$ and $\M{H}(\0)_{|_A}=C_{|_A}$. For $A=(a_1,...,a_m)\in X^m$ and $J=\{j_1,...,j_n\}\s\{1,...,m\}$, let $A_{|_J}$ denote $\{a_{j_1},...,a_{j_n}\}$.

There are $\binom{m}{n}$ many subsets of $A$ of size $n$, hence fixing $J$ a subset of $\{1,...,m\}$ of size $n$, the probability we wish to bound from above is at most
\begin{align*}
&P^m(\{A\in X^m: \exists I\s\{1,...,m\}\text{ of size }n\text{ where }\\ 
&~~~~~~~~~~~~~~~~~~~~~~~~~~~~~~~~~~~~~~~~~~~~~~~~~~~~~P(\HH(A_{|_I})\triangle C)>\varepsilon\text{ and }\M{H}(A_{|_I})_{|_A}=C_{|_A} \})\\
&=\binom{m}{n}P^m(\{A\in X^m: P(\HH(A_{|_J})\triangle C)>\varepsilon\text{ and }\M{H}(A_{|_J})_{|_A}=C_{|_A} \}).
\end{align*}

Since permuting $J$ to some other subset of size $n
$ in $\{1,...,m\}$ does not affect the above probability, we can assume $J=\{1,...,n\}$. 

We will prove at this point that
$\{A\in X^m: P(\HH(A_{|_J})\triangle C)>\varepsilon\text{ and }\M{H}(A_{|_J})_{|_A}=C_{|_A} \}$ is measurable due to the hypothesis that $\HH$ satisfies (M5): Let $1\leq p<q\leq m$ and let $\pi_{p,q}$ be the (measurable) function from $X^m$ to $X^{p+1}$ mapping $(x_1,...,x_m)\mapsto(x_1,...,x_p,x_q)$. By (M5) and the measurability of $C$ we have
\begin{align*}
&\{A\in X^m: \M{H}(A_{|_J})_{|_A}=C_{|_A} \}=\{A\in X^m:A\in((\M{H}(A_{|_J})\triangle C)^c)^m\}\\
&=\bigcap_{q=1}^m\{A\in X^m:(\M{H}(A_{|_J})\triangle C)^c)(A_{|_{\{q\}}})=1\}\\
&=\bigcap_{q=1}^m\pi_{n,q}^{-1}(\{(x_1,...,x_{n+1})\in X^{n+1}:(\M{H}(\{x_1,...,x_n\})\triangle C)^c)(x_{n+1})=1\})\in \A^m.
\end{align*}
Also $\{A\in X^m: P(\HH(A_{|_J})\triangle C)>\varepsilon\}$ is measurable since 
\begin{align*}
&B=\{(x_1,...,x_{m+1})\in X^{m+1}:(\M{H}(\{x_1,...,x_{n}\})\triangle C)^c(x_{m+1})=1\}\\
&=\pi_{n,m+1}^{-1}(\{(x_1,...,x_{n+1})\in X^{n+1}:(\M{H}(\{x_1,...,x_{n}\})\triangle C)^c(x_{n+1})=1\})
\end{align*}
is measurable by (M5) and the measurability of $C$, and a straightforward application of Fubini's theorem gives us that the map
\begin{align*}
&(x_1,...,x_m)\mapsto \int_X\chi_B(x_1,...,x_{m+1})dP(x_{m+1})=\\
&=P(\{y:(x_1,...,x_m,y)\in B\})\\
&=P(\{y:y\in\M{H}(\{x_1,...,x_{n}\})\triangle C)^c\})\\
&=P(\M{H}(\{x_1,...,x_{n}\})\triangle C)^c)
\end{align*}
is measurable.

Now let 
\begin{align*}
E_C&:=\{A\in X^{m}:\M{H}(A_{|_J})_{|_{A_{|_{\{n+1,...,m\}}}}}=C_{|_{A_{|_{\{n+1,...,m\}}}}}\}\\
&=\{A\in X^m:A_{|_{\{n+1,...,m\}}}\in((\M{H}(A_{|_J})\triangle C)^c)^{m-n}\}\\
&=\bigcap_{q=n+1}^m\{A\in X^m:(\M{H}(A_{|_J})\triangle C)^c)(A_{|_{\{q\}}})=1\}\\
&=\bigcap_{q=n+1}^m\pi_{n,q}^{-1}(\{(x_1,...,x_{n+1})\in X^{n+1}:(\M{H}(\{x_1,...,x_n\})\triangle C)^c)(x_{n+1})=1\})\in \A^m.
\end{align*}
and 
\begin{align*}
E_\varepsilon&:=\{A\in X^n:P(\M{H}(A)\triangle C)>\varepsilon\}\\
&=\{A\in X^n:P((\M{H}(A)\triangle C)^c)\leq (1-\varepsilon)\}. 
\end{align*}
$E_\varepsilon$ is measurable since $B=\{(x_1,...,x_{n+1})\in X^{n+1}:(\M{H}(\{x_1,...,x_{n}\})\triangle C)^c(x_{n+1})=1\}$ is measurable by (M5) and the measurability of $C$, and a straightforward application of Fubini's theorem gives us that the map
\begin{align*}
&(x_1,...,x_n)\mapsto \int_X\chi_B(x_1,...,x_{n+1})dP(x_{n+1})=\\
&=P(\{y:(x_1,...,x_n,y)\in B\})\\
&=P(\{y:y\in\M{H}(\{x_1,...,x_{n}\})\triangle C)^c\})\\
&=P(\M{H}(\{x_1,...,x_{n}\})\triangle C)^c)
\end{align*}
is measurable.\\
We have that
\begin{align*}
&P^m(\{A\in X^m: P(\HH(A_{|_J})\triangle C)>\varepsilon\text{ and }\M{H}(A_{|_J})_{|_A}=C_{|_A} \})\\
&\leq P^m(\{A\in X^m: P(\HH(A_{|_J})\triangle C)>\varepsilon\text{ and }\M{H}(A_{|_J})_{|_{A_{|_{\{n+1,...,m\}}}}}=C_{|_{A_{|_{\{n+1,...,m\}}}}} \})\\
&= P^m(E_C\cap (E_\e\times X^{m-n})).
\end{align*}
By Fubini's theorem 
\begin{align*}
&P^m(E_C\cap (E_\e\times X^{m-n}))=\int_{E_\e\times X^{m-n}}\chi_{E_C}(x_1,...,x_m)dP^m\\
&=\int_{E_\e}\left(\int_{X^{m-n}}\chi_{E_C}(x_1,...,x_m)dP^{m-n}\right)dP^n.
\end{align*}
Now
$$(x_1,...,x_n)\times X^{m-n}\cap E_C=(x_1,...,x_n)\times\{A\in X^{m-n}:A_{|_{\{n+1,...,m\}}}\in((\M{H}(A_{|_J})\triangle C)^c)^{m-n}\}$$
and since $(x_1,...,x_n)\in E_\e$, the inner integral is at most $(1-\e)^{m-n}$ and so 
$$P^m(E_C\cap (E_\e\times X^{m-n}))\leq(1-\e)^{m-n}.$$ 
Therefore summing over all subsets $J$ of $\{1,...,m\}$ of size at most $d$, the probability that $A\s X$, $|A|=m\geq d$, contains a subset $\0$ of size at most $d$ such that $P(\M{H}(\0)\triangle C)>\varepsilon >0$ and $\M{H}(\0)_{|_A}=C_{|_A}$, is at most 
$$\sum_{n=0}^d\binom{m}{n}(1-\varepsilon)^{m-n}.$$

\end{proof}
\begin{lem}[\cite{Floyd95samplecompression}]
Let $0<\e\leq1,\ 0<\delta\leq1$ and $m,d$ positive integers. If there is $0<\beta<1$ where 
$$ m\geq\frac{1}{1-\beta}\left(\frac{1}{\e}\ln(\frac{1}{\delta})+d+\frac{d}{\e}\ln(\frac{1}{\beta\e})\right),$$ then 
$$\sum_{n=0}^d\binom{m}{n}(1-\varepsilon)^{m-n}\leq\delta.$$
\end{lem}
\begin{proof}
Let $\e,\delta,\beta,m,d$ be as in the statement of the lemma. Then 
$$\frac{1}{\e}\ln(\frac{1}{\delta})+d+\frac{d}{\e}\left(-\ln(d)+\ln(\frac{d}{\beta\e})-1+\frac{\beta\e}{d}m+1\right)$$
$$=\beta m + \frac{1}{\e}\ln(\frac{1}{\delta})+d+\frac{d}{\e}\ln(\frac{1}{\beta\e})\leq m.$$
We will use the fact that $\ln(m)\leq-\ln(\alpha)-1+\alpha m$ for all $\alpha>0$. 
With $\alpha=\frac{\beta\e}{d}$ we get $$\ln(m)\leq\ln(\frac{d}{\beta\e})-1+\frac{\beta\e}{d}m$$ thus 
$$\frac{1}{\e}\ln(\frac{1}{\delta})+d+\frac{d}{\e}\left(-\ln(d)+\ln(m)+1\right)$$
$$\leq\frac{1}{\e}\ln(\frac{1}{\delta})+d+\frac{d}{\e}\left(-\ln(d)+\ln(\frac{d}{\beta\e})-1+\frac{\beta\e}{d}m+1\right)\leq m.$$ Hence we have $$\ln(\frac{1}{\delta})+d(-\ln(d)+\ln(m)+1)\leq \e(m-d)$$ and so 
$$\left(\frac{em}{d}\right)^d\leq e^{\e(m-d)}\delta.$$ Therefore since $m\geq d$, $$\sum_{i=0}^d\binom{m}{i}(1-\varepsilon)^{m-i}\leq\binom{m}{\leq d}(1-\varepsilon)^{m-d}\leq\left(\frac{em}{d}\right)^d(1-\varepsilon)^{m-d}\leq\left(\frac{em}{d}\right)^de^{-\e(m-d)}\leq\delta.$$
\end{proof}

Theorem 3.2.1 and the last lemma lead to bounds for the sample complexity of sample 
compression schemes. We illustrate in Figure 3.1 on the next page, that these bounds for sample compression schemes of size $d$ are better than the ones for general consistent learning rules on a space with VC dimension $d$. Note that 0.05 is one of the standard choices for risk in statistics.
\begin{theo}[\cite{Floyd95samplecompression}]
If $(X,\CC)$ has a sample compression scheme $\HH$ of size $d$ satisfying (M5), then for $0<\e\leq1,\ 0<\delta\leq1$, if $\LL_{\HH}$ is a learning rule (if $\LL_{\HH}$ satisfies (M2)) it has sample complexity at most
$$m_{\LL_{\HH}}(\e,\delta)\leq\frac{1}{1-\beta}\left(\frac{1}{\e}\ln(\frac{1}{\delta})+d+\frac{d}{\e}\ln(\frac{1}{\beta\e})\right),$$
for any $0<\beta<1$.
\end{theo}
\begin{proof}
Let $\HH$ be as in the statement of the theorem, fix $0<\e\leq1,\ 0<\delta\leq1,\ 0<\beta<1,\ C\in\CC,\ P\in \m{P}$, and let 
$$m\geq\frac{1}{1-\beta}\left(\frac{1}{\e}\ln(\frac{1}{\delta})+d+\frac{d}{\e}\ln(\frac{1}{\beta\e})\right).$$
Since $\HH$ is consistent and satisfies (M5), by theorem 3.2.1
\begin{align*}
&P^m(\{A\in X^m:P(\LL_{\HH}(A,C_{|_A})\triangle C)>\e\})\\
&\leq P^m(\{A\in X^m:\exists \0\in[A]^{\leq d} \text{ where } P(\M{H}(\0)\triangle C)>\varepsilon\})\\
&=P^m(\{A\in X^m:\exists \0\in[A]^{\leq d} \text{ where } P(\M{H}(\0)\triangle C)>\varepsilon \text{ and }\M{H}(\0)_{|_A}=C_{|_A}\})\\
&\leq\sum_{n=0}^d\binom{m}{n}(1-\varepsilon)^{m-n},
\end{align*}
and by lemma 3.2.2 
$$\sum_{n=0}^d\binom{m}{n}(1-\varepsilon)^{m-n}\leq\delta.$$
\end{proof}

\begin{figure}
\centering
\includegraphics{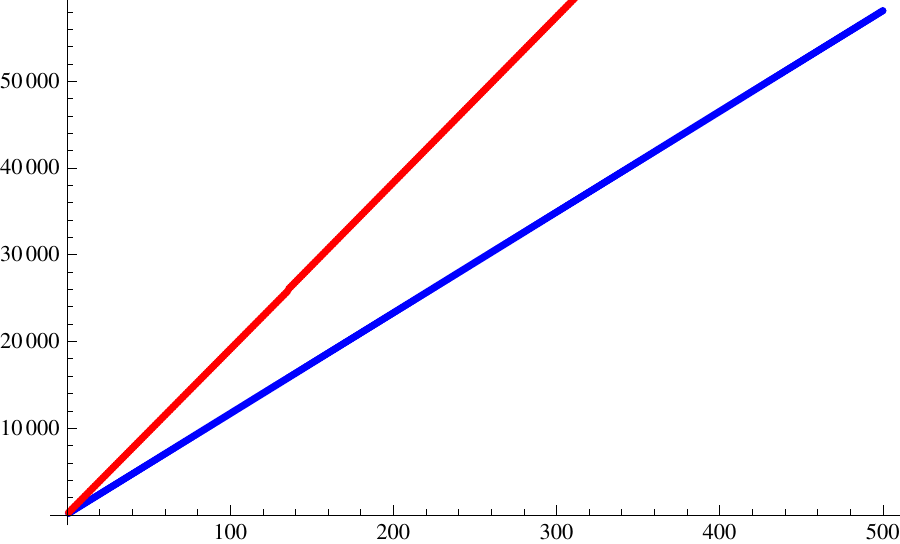}
\caption{
Given $ \e=0.05,\ \delta=0.05$, we plot a graph with $d$ on the x-axis and:\newline
$f(d)=\frac{1}{1-\beta}\left(\frac{1}{\e}\ln(\frac{1}{\delta})+d+\frac{d}{\e}\ln(\frac{1}{\beta\e})\right)$
in blue, and \newline
$g(d)=\frac{1}{1-\beta}\left(\frac{1}{\e}\ln(\frac{2}{\delta})+\frac{2d\ln2}{\e}+\frac{d}{\e}\ln(\frac{1}{\e \beta^2})\right)$
in red, \newline
where we optimize $\beta$ in each function for each value of $d$.
}
\label{fig:1}
\end{figure}
~~~~~~

~~~~~~

The sample complexity for our copy sample compression schemes can be bounded in similar fashion to that for sample compression schemes. The proofs are very similar to those for sample compression schemes and can be found in the appendix B.2.
\begin{theo}
Let $P$ be any probability measure on a measurable space $(X,\mathfrak{A})$, $C$ a concept in $\m{C}\s \mathfrak {A}$, $\{n_i\}_{i=0}^{k}\s\NN$, and $\m{H}$ any function from\\
$\bigcup_{i\in\{j\in\NN:0\leq j\leq k,\ n_j\neq 0\}}([X]^{= i}\times\{1,...,n_i\})$ to $2^X$ satisfying measurability condition (M5). Then the probability that $A\s X$, $|A|=m\geq k$, contains a subset $\0$ of size at most $k$, and $l\in\{1,...,n_{|\0|}\}$ such that $P(\M{H}(\0\times l)\triangle C)>\varepsilon >0$ and $\M{H}(\0\times l)_{|_A}=C_{|_A}$, is at most 
$$\sum_{i=0}^k n_i\binom{m}{i}(1-\varepsilon)^{m-i}.$$
\end{theo}

\begin{lem}
Let $0<\e\leq1,\ 0<\delta\leq1$, $m,k$ positive integers, and $n=\max(\{n_i\}_{i=0}^{d})$. if there is $0<\beta<1$ where 
$$m\geq\frac{1}{1-\beta}\left(\frac{1}{\e}\ln(\frac{n}{\delta})+k+\frac{k}{\e}\ln(\frac{1}{\beta\e})\right),$$ then 
$$\sum_{i=0}^kn_i\binom{m}{i}(1-\varepsilon)^{m-i}\leq\delta.$$
\end{lem}

Theorem 3.2.4 and the last lemma lead to bounds for the sample complexity of copy sample 
compression schemes.
\begin{theo}
If $(X,\CC)$ has a $\{n_i\}_{i=0}^{k}$-copy sample compression scheme $\HH$ of size $k$ satisfying (M5), and $n=\max(\{n_i\}_{i=0}^{d})$, then for $0<\e\leq1,\ 0<\delta\leq1$, if $\LL_{\HH}$ is a learning rule (if $\LL_{\HH}$ satisfies (M2)) it has sample complexity at most
$$m_{\LL_{\HH}}(\e,\delta)\leq\frac{1}{1-\beta}\left(\frac{1}{\e}\ln(\frac{n}{\delta})+k+\frac{k}{\e}\ln(\frac{1}{\beta\e})\right),$$
for any $0<\beta<1$.
\end{theo}

Compared to the bounds for sample complexity of sample compression schemes, copy sample compression schemes may be better in some instances. For example, in any concept space with $|X|=884$ and a sample compression scheme of size $7$ there exists (by proposition 2.4.4) a $18418$-copy sample compression scheme of size $5$. In this case, if we wish to learn with
accuracy $\e=0.05$ and risk $\delta=0.05$, the previous theorem guarantees the $18418$-copy sample compression scheme of size $5$ will achieve this with sample size $879$, but the bounds for sample complexity of sample compression schemes (for all $\beta\in (0,1)$) exceed $884$.

\cleardoublepage

\chapter[Measurability of Sample Compression Schemes and Learning Rules]{Measurability of Sample Compression Schemes and Learning Rules}

\section{Measurability of Compression Schemes}
~~~

Given a sample compression scheme $\HH$ of size $d$ and $\0\in[X]^{\leq d}$, the corresponding hypothesis $H(\0)$ is not necessarily measurable with respect to a given sigma algebra on $X$. If we are creating our compression scheme via the compactness theorem we would like to be able to see when the resulting compression scheme will have $H(\0)$ measurable $\forall\0\in[X]^{\leq d}$. This condition is necessary for sample compression schemes to be considered as a learning rule, but it is not sufficient and in the appendix B.1 we will discuss a sufficient condition ``(M5)" which allowed for the sample complexity bounds in theorem 3.2.3.

\begin{notation} 
Let $(X,\mathfrak{A})$ be a measurable space, $Pr(X)$ the set of all probability measures on $(X,\mathfrak{A})$, and let $\mathfrak{A}(X)$ be the set of all real valued bounded functions on $X$ which are
measurable with respect to $\mathfrak{A}$.
\end{notation}
In the future we will fix $(X,\mathfrak{A})$ to be our measurable space unless otherwise stated.

The following remark is obvious.
\begin{rmk}
Note that if a net of $\{0,1\}$-valued functions $\{f_{\alpha}\}_{\alpha\in I}$ converges pointwise to some $f\in\RR^X$, then 
$$\forall A\in [X]^{<\infty}\ \exists \alpha_0\in I\ \forall\alpha\in I,\ \alpha\geq\alpha_0\ \Rightarrow\ f_{\alpha |_A}=f_{|_A}$$ and in particular $f\in2^X$.
\end{rmk}

\begin{lem}[\cite{dudley1999uniform}]
Let $\m{F}\s2^X$ be $d$-maximal for some $d$. Then $\m{F}$ is closed under the topology of pointwise convergence of nets on $\RR^X$, the product topology on $\RR^X$. That is, if $\{f_{\alpha}\}_{\alpha\in I}$ is a net of functions in $\m{F}$ converging pointwise to some real valued function $f$, then $f\in \m{F}$ (for the remainder of this section when we write $f_{\alpha}\rightarrow f$, we mean that the convergence is pointwise unless mentioned otherwise).
\end{lem}
\begin{proof}
Let $\{f_{\alpha}\}_{\alpha\in I}$ be a net in $\m{F}$ such that $f_{\alpha}\rightarrow f\in\RR^X$. For $A\in [X]^{<\infty}$, by \\remark 4.1.2 $f\in2^X$. Also $\exists \alpha\in I$ such that $f_{\alpha |_A}
=f_{|_A}$ and so 
$$f_{|_A}\in\{f_{\alpha |_A}\}_{\alpha\in I}\s\m{F}_{|_A}.$$ 
Now suppose $f\notin\m{F}$. Since $\m{F}$ is $d$-maximal, $\vc(\m{F}\cup\{f\})>d$ and so $\exists A\in [X]^{<\infty}$ of cardinality $d+1$ such that $2^A=\m{F}_{|_A}\cup \{f_{|_A}\}=\m{F}_{|_A}$, contradicting $\vc(\m{F})=d$.
\end{proof}

One would wonder whether in a measurable space $(X,\mathfrak{A})$ with $\m{C}\s\mathfrak{A}$ and\\$\vc(\m{C})<\infty$, the closure of $\m{C}$ in $\RR^X$ lies within $\mathfrak{A}$. The answer is negative in general as shown by our following example (the definition of universally measurable is found in appendix A.1):

\begin{egg}
Let $(X,\m{B})$ be an uncountable standard Borel space and $\mathfrak{A}$ the sigma algebra of universally Borel measurable sets. Let $\kappa$ be the cardinal
$$\kappa=\min\{\kappa':\exists A\s X \text{ non }\mathfrak{A}\text{-measurable and } |A|=\kappa'\}.$$
Of course, $\aleph_0<\kappa\leq \mathfrak{c}$, but the value of $\kappa$ cannot be specified in ZFC and would require additional set theoretic axioms; for example under Martin's Axiom $\kappa=\mathfrak{c}$.
Fix $Y\s X$ non $\mathfrak{A}$-measurable with $|Y|=\kappa$.
Well order $Y$ with $\prec$ such that each initial segment $I_x^{\prec}=\{y:y\preceq x\}$ of $Y$ has cardinality strictly less than $Y$ (in other words fix a minimal well ordering on $Y$). By the definition of $\kappa$, each $I_x^{\prec}$ is universally measurable. Now set for $x\in Y$, 
$$f_x=\chi_{I_x^{\prec}}.$$
One has $\{f_x\}_{x\in Y}\s\mathfrak{A}$ and $\vc(\{f_x\}_{x\in Y})=1$ ($(Y,\{f_x\}_{x\in Y})$ is $1$-maximum). We have that $f_x\rightarrow\chi_Y$ as a net with $Y$ directed by $\prec$ because for every $A\in[X]^{<\infty}$, letting $x\succeq a=\max_\prec A\cap Y$ (or pick any $a \in Y$ if $A\cap Y=\emptyset$), we have 
$$f_{x |_A}=\chi_{I_x^{\prec}\cap Y |_A}=\chi_{I_x^{\prec}\cap A\cap Y |_A}=\chi_{A\cap Y |_A}=\chi_{Y |_A}.$$

\end{egg}
	
We can use the previous example to show that the compactness theorem, using sample compression schemes on finite subdomains, can create a sample compression scheme returning a nonmeasurable hypothesis.
\begin{egg}
In the previous example let 
$$\CC=\{I_x^{\prec}: I_x^{\prec} \text{ is an initial segment of} Y\}.$$
Note that $(Y,\CC)$ is $1$-maximum
and by example 2.1.4, 
$$\m{H}(\{x\})=I_x^{\prec}\setminus\{x\},\ \m{H}(\emptyset)=Y$$ is a sample compression scheme of size $1$ for $(Y,\CC)$. 
We have that $\HH(\emptyset)=Y$ is not measurable even though $\CC$ consists only of measurable sets.\\
Note however by example 2.1.4
$$\m{H}(\{x\})=I_x^{\prec},\ \m{H}(\emptyset)=\emptyset$$
is also a sample compression scheme of size 1 for $(Y,\CC)$, and has measurable hypotheses. However, in this thesis we are unable to find a concept space with VC dimension $d$ and measurable concepts, in which any sample compression scheme of size $d$ must have a nonmeasurable hypothesis.
\end{egg}

\begin{lem}
Let $(X,\CC)$ be a concept space. For every finite subspace $(A,\CC\sqcap A)$, let $\m{H}_A$ be a compression scheme of size $d$ on $(A,\CC\sqcap A)$, let $\m{H}$ be the compression scheme of size $d$ on $X$ defined by $\{\m{H}_A\}_{A\in[X]^{< \infty}}$ via the compactness theorem (theorem 2.2.1), and let $\0\in[X]^{\leq
 d}$. If for all $A\in[X]^{<\infty}$ with $\0\s A$, a function $f_{A,\0}\in2^X$ is any extension of $\m{H}_A(\0)$ to $X$, we have that $\M{H}(\0)$ is a cluster point of the net $\{f_{A,\0}\}_{\0\s A\in[X]^{<\infty}}$ indexed by $A$ (where $\{A\in[X]^{<\infty}:\0\s A\}$ is directed by inclusion) in the topology of pointwise convergence; the product topology on $2^X$.
\end{lem}
\begin{proof}
Let $\0\in[X]^{\leq d}$, let $\{f_{A,\0}\}_{\0\s A\in[X]^{<\infty}}$, $\{\m{H}_A\}_{A\in[X]^{< \infty}}$, $\m{H}$ be as in the statement of the theorem, and let $\mathfrak{U}$ be the ultrafilter on $[X]^{< \infty}$ defined as in our proof of the compactness theorem.
Let $x\in X$, 
$$\m{A}_x=\{A\in[X]^{< \infty}: \0\cup\{x\}\s A,\ \m{H}_A(\0)(x)=\m{H}(\0)(x)\}.$$ 
$\m{A}_x\in\mathfrak{U}$ because 
$$\m{H}(\0)(x)=1\Rightarrow\m{A}_x=\{A\in[X]^{< \infty}: \0\cup\{x\}\s A,\ \m{H}_A(\0)(x)=1=\m{H}(\0)(x)\}\in\mathfrak{U}$$ and 
$$\m{H}(\0)(x)=0\Rightarrow\m{A}_x^c=\{A\in[X]^{< \infty}: \0\cup\{x\}\s A,\ \m{H}_A(\0)(x)=1\neq\m{H}(\0)(x)\}\notin\mathfrak{U}\Rightarrow\m{A}_x\in\mathfrak{U}.$$ 
Now fix a basic open neighbourhood 
$$U=\{\m{H}(\0)(x_1)\}_{x_1}\times...\times\{\m{H}(\0)(x_n)\}_{x_n}\times\prod_{x\notin\{x_1,...,x_n\}}\{0,1\}_x$$ of $\m{H}(\0)$ in $2^X$, and let $B\in[X]^{< \infty}$. There is 
\begin{align*}
&B'\in\bigcap_{i=1}^{i=n}\m{A}_{x_i}\cap\bigcap_{b\in B}\m{A}_b=\\
&=\{A\in[X]^{< \infty}: \0\cup\{x_1,...,x_n\}\cup B\s A,\ \m{H}_A(\0)_{|_{\{x_1,...,x_n\}}}=\m{H}(\0)_{|_{\{x_1,...,x_n\}}}\}\in\mathfrak{U},
\end{align*}
hence $\0\cup\{x_1,...,x_n\}\cup B\s B'$ and $\m{H}_{B'}(\0)_{|_{\{x_1,...,x_n\}}}=\m{H}(\0)_{|_{\{x_1,...,x_n\}}}$. There is $f_{B',\0}$ such that $f_{B',\0|_{B'}}=\m{H}_{B'}(\0)$ thus $f_{B',\0|_{\{x_1,...,x_n\}}}=\m{H}(\0)_{|_{\{x_1,...,x_n\}}}$ and therefore $f_{B',\0}\in U$.
\end{proof}

\begin{lem}[Measurable Compactness Lemma]
Let $(X,\mathfrak{A})$ be a measurable space. $(X,\CC)$ has a sample compression scheme of size $d$ with measurable hypotheses if and only if every finite subspace $(A,\CC\sqcap A)$ of $(X,\CC)$ has a sample compression scheme of size $d$, $\m{H}_A$, where for all $ \0 \in [X]^{\leq d}$, there is $d_{\0}\in\NN$ and $\m{M}_{\0}\s\mathfrak{A}$ $d_{\0}$-maximal, such that $$\{\m{H}_A(\0)\}_{\0\s A\in[X]^{< \infty}}\s\{f_{|_A}:f\in\mathcal{M}_{\0},\ \0\s A\in[X]^{< \infty}\}.$$
\end{lem}
\begin{proof}
\emph{Necessity}: Let $\m{H}$ be a sample compression scheme of size $d$ with measurable hypotheses on $(X,\CC)$, and let $\0\in[X]^{\leq d}$. Now set $\m{M}_{\0}=\{\m{H}(\0)\}$, noting that $\m{M}_{\0}$ is $0$-maximal. Let $ A\s X$ be a finite subset, we have that 
$$\m{H}_A:[A]^{\leq d}\rightarrow 2^A,\ \0\mapsto\m{H}(\0)_{|_A}\in\m{M}_{\0|_A}$$ is a sample compression scheme of size $d$ on $A$.
\\
\emph{Sufficiency}: Let $\{\m{H}_A\}_{A\in[X]^{< \infty}}$, $\{d_\0\}_{\0\in[X]^{\leq d}}$, $\{\m{M}_{\0}\}_{\0\in[X]^{\leq d}}$ be as in the statement of the theorem, and let $\m{H}$ be the sample compression scheme of size $d$ on $(X,\CC)$ defined by $\{\m{H}_A\}_{A\in[X]^{< \infty}}$ via the compactness theorem (theorem 2.2.1). Let $\0\in[X]^{\leq d}$ be given. By lemma 4.1.3, $\m{M}_
{\0}$ is closed in $2^X$ since it is maximal. Since $\m{H}_A(\0)\in\m{M}_{\0|_A}$ for all $\0\s A\in[X]^{<\infty}$,
by lemma 4.1.6, 
$$\m{H}(\0)\in\overline{\m{M}_{\0}}=\m{M}_{\0}\s\mathfrak{A}.$$
\end{proof}
\begin{theo}
If $(X,\mathfrak{A})$ is a measurable space and $\m{C}\s\mathfrak{A}$ is a $d$-maximum and $d$-maximal class, then
$(X,\CC)$ has a sample compression scheme of size $d$ with measurable hypotheses. In this case, the sample compression scheme maps to concepts in $\m{C}$.
\end{theo}
\begin{proof}
$(X,\CC)$ is $d$-maximum implies that: $(X,\CC)$ has a sample compression scheme $\m{H}$ of size $d$, and every finite subspace $(A,\CC\sqcap A)$ is $d$-maximum. For every finite subspace $(A,\CC\sqcap A)$, we have $\m{H}_A:[A]^{\leq d}\rightarrow 2^A$, $\0\mapsto\m{H}(\0)_{|_A}$ is a sample compression scheme of size $d$ on $A$. Since $(A,\CC\sqcap A)$ is $d$-maximum, for all $\0\in[A]^{\leq d}$, $\m{H}_A(\0)\in\CC\sqcap A$. This shows that the hypothesis of lemma 4.1.7 in the 
converse direction holds for $\m{M}_{\0}=\CC$ and so $(X,\CC)$ has a sample compression scheme of size $d$ with measurable hypotheses. In particular, from the proof of the theorem, we see that 
$$\{\m{H}(\0):\0\in[X]^{\leq d}\}\s\bigcup_{\0\in[X]^{\leq d}}\m{M}_{\0}=\M{C}.$$
\end{proof}

\begin{notation}
$\bold{C}(X)$ is notation for the set of continuous functions from $X$ to $\RR$, where $\RR$ has the usual topology and the topology on $X$ will be clear in the context.
\end{notation}

The following theorem is an important result from \cite{Bourgain_Fremlin_Talagrand_1978} (theorem 2F p. 854 ), which will allow us to drop the maximality condition of the corollary in some circumstances.

\begin{theo}
Let $X$ be any Hausdorff space and $\m{C}\s \bold{C}(X)$ a pointwise bounded set. Then $\m{C}$ is relatively compact in the collection of universally Borel measurable real valued functions on $X$ under the topology of pointwise convergence if and only if for every $K\s X$ compact, every $\{C_n\}_{n\in\NN}\s\m{C}$, and every $\alpha<\beta$ in $\RR$, there exists $I\s\NN$ such that 
$$\{x\in K: C_n(x)\leq\alpha\ \forall n\in I,\ C_n(x)\geq\beta\ \forall n\in \NN\setminus I\}=\emptyset.$$
\end{theo}
\begin{lem}
Let $X$ be any Hausdorff space with a $d$-maximum concept class $\m{C}$ consisting of clopen sets. Then $(X,\CC)$ has a sample compression scheme of size $d$ with universally Borel measurable hypotheses.
\end{lem}
\begin{proof}
Clearly $\m{C}\s \bold{C}(X)$, and $\CC$ is a pointwise bounded set of functions. 
Suppose there is $\{C_n\}_{n\in\NN}\s\m{C}$ such that $\forall I\s\NN$, there is $x\in X$ where $x\notin C_n\ \forall n\in I$ and $x\in C_n\ \forall n\in \NN\setminus I$. This implies that $\vc((X,\CC)^*)=\infty$, which contradicts $\vc(X,\CC)<\infty$. Thus for every $\{C_n\}_{n\in\NN}\s\m{C}$, there exists $I\s\NN$ such that 
$$\{x\in X: C_n(x)=0\ \forall n\in I,\ C_n(x)=1\ \forall n\in \NN\setminus I\}=\emptyset,$$ and so by the previous theorem $\m{C}$ is relatively compact in the collection of universally Borel measurable real valued functions on $X$. 
Since $\CC$ is $d$-maximum, for all $\0\in[X]^{\leq d},\ \0\s A\in[X]^{< \infty}$, we have $\m{H}_A(\0)\in\CC_{|_A}$, and so using lemma 4.1.6 similarly to lemma 4.1.7
we have that $(X,\CC)$ has a sample compression scheme $\m{H}$ of size $d$ such that $\{\m{H}(\0):\0\in[X]^{\leq d}\}\s\overline{\M{C}}$. Therefore $\{\m{H}(\0):\0\in[X]^{\leq d}\}$ consists of universally Borel measurable sets.
\end{proof}

The following definitions can be found in reference \cite{pollard1984convergence} pp. 38-39.
\begin{defn}
A family $\CC$ of functions has a subset $\m{D}$ which is {\bfseries universally dense}\index{Universally dense} in $\M{C}$ if every $C\in\M{C}$ is the pointwise limit of a sequence in $\m{D}$.\\
A concept space $(X,\CC)$ is {\bfseries universally separable}\index{Universally separable} if there exists a countable universally dense subset $\m{D}$ of $\M{C}$.
\end{defn}
\begin{rmk}
$\m{D}'$ is universally dense in $\m{D}$, is equivalent to $\m{D}'$ is dense in $\m{D}$ in the topology generated by the $L^1(P)$ seminorm for every $P\in Pr(X)$.
\end{rmk} 
\begin{rmk}
Note that for every $A\in[X]^{<\infty}$, if $\m{D}$ is universally dense in $\M{C}$, then $\m{D}\sqcap A=\M{C}\sqcap A$. This implies that $\M{H}$ is a sample compression scheme of size $d$ on $(X,\m{D})$ if and only if it is a sample compression scheme of size $d$ on $(X,\m{C})$. Therefore, to prove something about the sample compression schemes for a universally separable concept class, we only need to consider sample compression schemes for countable concept classes.
\end{rmk}

We will use the following classical result in descriptive set theory (which can be found in \cite{kechris1995classical} pp. 83) about standard Borel spaces to get that all $d$-maximum universally separable concept spaces, in which the countable universally dense subset consists of Borel sets, have a sample compression scheme of size $d$ with universally Borel measurable hypotheses:
\begin{lem}
Let $X$ be a standard Borel space, and let $\m{C}$ be a countable family of Borel sets. There is a refinement of the topology on $X$ to a Polish topology generating the same Borel sigma algebra in which $\m{C}$ consists of clopen sets.
\end{lem}

By combining the last two lemmas (lemmas 4.1.11 and 4.1.15) and noting remark 4.1.14, we get the 
 following potentially useful result:
\begin{theo}
Let $X$ be a standard Borel space with a $d$-maximum and universally separable concept class $\m{C}$. Then $(X,\CC)$ has a sample compression scheme of size $d$ with universally Borel measurable hypotheses.
\end{theo}

\cleardoublepage

\nonumchapter{Open Questions}
~~~

We have raised the following open question: even assuming that a sample compression scheme exists for every finite concept subspace of a given space, can we conclude that such a scheme consisting of measurable hypotheses will exist for a concept class consisting of Borel sets on a standard Borel domain? A stronger version of this question can also be asked for the measurability condition (M5), that one needs for the proof of the main sample complexity bound for sample compression schemes: Given that a sample compression scheme exists for every finite concept subspace of a given space, can we conclude that such a scheme also satisfying (M5) will exist for a concept class consisting of Borel sets on a standard Borel domain? Finally, one can ask if there is validity of either question under a stronger assumption on the concept class: e.g., image admissible Suslin, or universally separable.

It is an open question of whether or not there exists a sample compression scheme of size $O(d)$ for a concept space of VC dimension $d$. We may ask similar but weaker questions of copy sample compression schemes that may help in clarifying methods to prove, or disprove, the original question. For example: Given a concept space of VC dimension $d$ is there an $O(d)$-copy sample compression scheme of size $O(d)$?

\cleardoublepage








\appendix

\chapter{}

\section{Measure Theory Preliminaries}
~~

The following are some preliminary measure theoretic definitions, and can be found in \cite{billingsley2012probability} and \cite{kechris1995classical}.
\begin{defn}
A {\bfseries sigma algebra}\index{Sigma algebra} $\A$ on a set $X$, is a set of subsets of $X$ satisfying:
\begin{description}
\item (1) $X\in\A$
\item (2) $A\in\A$ implies $A^c\in\A$
\item (3) $\{A_i\}_{i=0}^{\infty}\s\A$ implies $\bigcup_{i=0}^{\infty}A_i\in\A$.

\end{description}
A {\bfseries measurable space}\index{Measurable space} $(X,\A)$ is a set $X$ equipped with a sigma algebra $\A$.
\end{defn}

It is easy to see that an intersection of sigma algebras is also a sigma algebra.
\begin{defn}
A {\bfseries sigma algebra generated}\index{Sigma algebra generated} by a family $\mathfrak{B}\s 2^X$, denoted $\0(\mathfrak{B})$, is the intersection of all sigma algebras on $X$ containing $\mathfrak{B}$.
\end{defn}
\begin{defn}
A {\bfseries measure}\index{Measure} on measurable space $(X,\A)$, is a function \\$\mu:\A\rightarrow[0,\infty]$ satisfying:
\begin{description}
\item (1) $\mu(\emptyset)=0$
\item (2) if $\{A_i\}_{i\in I}\s\A$ is a countable pairwise disjoint family, then $$\mu(\bigcup_{i\in I}A_i)=\sum_{i\in I}\mu(A_i).$$

\end{description}
A {\bfseries probability measure}\index{Probability measure} on measurable space $(X,\A)$, is a measure $P$ also satisfying $P(X)=1$. We denote the set of all probability measures on $(X,\A)$ by $Pr(X)$.
A {\bfseries measure space}\index{Measure space} $(X,\A,\mu)$ is a measurable space $(X,\A)$ equipped with a measure $\mu$.
\end{defn}

\begin{defn}
The {\bfseries completion}\index{Completion of a measure space} of a measure space $(X,\A,\mu)$ is the measure space $(X,\A_{\mu},\widehat{\mu})$ where 
$$\A_{\mu}=\0(\A\cup\{B\in 2^X:\exists A\in \A,\ B\s A\text{ and } \mu(A)=0\}),$$
and 
$$\widehat{\mu}:\A_{\mu}\rightarrow[0,\infty];B\mapsto \inf\{\mu(A):A\in\A,\ B\s A\}.$$
\end{defn}
It is true that $\widehat{\mu}$ is the unique measure extending $\mu$ to $\A_{\mu}$, and any $B\in\A_{\mu}$ is of the form $A\cup C$ where $A\in\A,\ C\in\{Y\in 2^X:\exists Y'\in \A,\ Y\s Y'\text{ and } \mu(Y')=0\}$.

\begin{defn}
For a measure space $(X,\A)$ the sigma algebra of {\bfseries universally measurable}\index{Universally measurable} subsets of $X$ (with respect to $\A$) is 
$$\A^*=\bigcap_{P\in Pr(X)}\A_{P}.$$
We say a set is universally measurable (with respect to $\A$), if it is an element of $\A^*$.
\end{defn}

\begin{defn}
The {\bfseries Borel sigma algebra}\index{Borel sigma algebra} of a topological space $(X,\m{T})$ is the sigma algebra $\m{B}=\0(\M{T})$.
A {\bfseries Borel space}\index{Borel space} is a measurable space $(X,\M{B})$ on a topological space, where the sigma algebra $\m{B}$ is the Borel sigma algebra on $X$.
\end{defn}

\begin{defn}
A {\bfseries Polish space}\index{Polish space} is a topological space $(X,\M{T})$ that is metrizable by a complete metric, and separable.
\end{defn}

\begin{defn}
A {\bfseries standard Borel space}\index{Standard Borel space} is a Borel space associated to a Polish topological space.
\end{defn}

Every uncountable Borel space $(X,\m{B})$ is isomorphic as a measure space to $[0,1]$ with the Borel sigma algebra (ie. there is a bijection $f:X\rightarrow[0,1]$ such that for every $A\s X$, $A$ is measurable in $X$ iff $f(A)$ is measurable in $[0,1]$) \cite{kechris1995classical}.


\section{Nets and Filters Preliminaries}
~~~

The following are some preliminary definition that can be found in \cite{munkres2000topology} and \cite{bell2006models}

\begin{defn}
A {\bfseries directed set}\index{Direct set} $I$ is a set equipped with a partial order such that 
$$\forall\alpha,\beta\in I,\ \exists\gamma\in I,\ \gamma\geq\beta\text{ and }\gamma\geq\alpha.$$
A subset $J$ of a directed set $I$ is {\bfseries cofinal}\index{Cofinal} in $I$ if
$$\forall\alpha\in I,\ \exists\beta\in J,\ \beta\geq\alpha.$$
\end{defn}
\begin{defn}
A {\bfseries net}\index{Net} in $X$ is a function mapping a directed set to $X$.
A {\bfseries subnet}\index{Subnet} of a net $f:I\rightarrow X$ is the function $f\circ g$, where $g$ is a nondecreasing function from a directed set $J$ to $I$ with $g(J)$ cofinal in $I$.
\end{defn}

\begin{defn}
A point $x$ of is a {\bfseries cluster point}\index{Cluster point} of a net $\{x_\alpha\}_{\alpha\in I}$ in a topological space, if for every neighbourhood $U$ of $x$ the set $\{\alpha\in I:x_\alpha\in U\}$ is cofinal in $I$.\\
A net $\{x_\alpha\}_{\alpha\in I}$ in a topological space {\bfseries converges}\index{Convergence of nets} to a point $x$ if for every neighbourhood $U$ of $x$ 
$$\exists \alpha\in I,\ \forall \beta\geq\alpha,\ x_{\beta}\in U.$$ 
\end{defn}

A point $x$ is a cluster point of a net if and only if there is a subnet converging to $x$.

\begin{defn}
A {\bfseries filter }\index{Filter} $\M{F}$ on a set $X$ is a set of subsets of $X$ satisfying:
\begin{description}
\item (1) $\emptyset\notin\M{F}$
\item (2) $A\in\M{F}$ and $A\s B$ imply $B\in\M{F}$
\item (3) $A,B\in\M{F}$ implies $A\cap B\in\M{F}$.
\end{description}
A {\bfseries filter base}\index{Filter base} $\M{F}$ on a set $X$ is a set of subsets of $X$ satisfying:
\begin{description}
\item (1) $\emptyset\notin\M{F}$
\item (2) $A,B\in\M{F}$ implies $A\cap B\in\M{F}$.
\end{description}
An {\bfseries ultrafilter}\index{Ultrafilter} $\mathfrak{U}$ on a set $X$ is a filter on $X$ where every $A\s X$ has either $A\in\mathfrak{U}$ or $A^c\in\mathfrak{U}$.
\end{defn}

Every filter base $\M{F}$ generates a unique filter that is equal to the intersection of all filters containing $\M{F}$, and every filter is contained in an ultrafilter as a consequence of Zorn's Lemma.

There are two types of ultrafilters, principal and free.
\begin{defn}
An ultrafilter $\mathfrak{U}$ on $X$ is {\bfseries principal}\index{Principal ultrafilter} if it has a least element under set inclusion, and is {\bfseries free}\index{Free ultrafilter} if it is not principal.
\end{defn}
Principal ultrafilters are of the form $\mathfrak{U}_A=\{B\in2^X:A\s B\}$ for some $A\s X$.

\cleardoublepage

\chapter{}

\section{Well Behaved Hypothesis Spaces and Other Forms of Measurability}
~~

The subject of this section will be additional conditions of measurability for learning rules. 

In the definition of a learning rule we had required that the hypothesis class is measurable, and that the function 
$$A\mapsto P(\LL(A,C_{|_A})\triangle C)$$ is measurable.
We will call these conditions (M2)\index{M1}\index{M2} and (M1) respectively. Note that (M1) implies (M2) and the consideration of chapter 4 was the (M2) condition for sample compression schemes.\\
Fix a measurable space $(X,\A)$.
\begin{notation}
Fix a set C, and an $\e>0$. Define
$$\m{D}_C=\{H\triangle C: H\in \mathfrak{H}\}.$$
When $\{C\}\cup\mathfrak{H}\s\mathfrak{A}$, for a measure $P\in Pr(X)$ define
$$\m{D}_C^\e=\{D\in\m{D}_C: P(D)>\e\}.$$
Note that for $m\in\NN$, $\A^m$ is notation for the product sigma algebra on $X^m$.
\end{notation}
\begin{defn}[\cite{Shawe-taylor93boundingsample}, \cite{Blumer:1989:LVD:76359.76371}]
Let $\{C\}\cup\hh\s\A$ and $\e>0$. Define 
$$Q^m_\e=\{A\in X^m:\exists D\in \m{D}_C^\e,\text{ such that } A\in (D^c)^m\}$$ and define
$$J^{m+k}_{\e,r}=$$
\begin{align*}
&=\{A\in X^{m+k}:\exists D\in \m{D}_C^\e, A_{|_{\{1,..,m\}}}\in (D^c)^m\text{ and }\exists I\s\{m+1,...,m+k\}, \\
&~~~~~~~~~~~~~~~~~~~~~~~~~~~~~~~~~~~~~~~~~~~~~~~~~~~~~~~~~~\text{such that }|I|\geq kr\e\text{ and }A_{|_I}\in D^{|I|} \}.
\end{align*}
We will say a hypothesis space satisfies {\bfseries(M3)}\index{M3} if:\\
 $$\text{ for all }m\geq 1,\ \e>0,\ C\in\A,\ P\in Pr(X,\A),$$ we have $$Q^m_\e\in\A^m,$$ and $$\text{ for all }m\geq \frac{4d}{\e},\ \e>0,\ k=m(\frac{\e r m}{d}-1),\ r=1-\sqrt{\frac{2}{\e k}},\ C\in\A,\ P\in Pr(X,\A)$$ such that $k$ is in $\NN$, we have $$J^{m+k}_{\e,r}\in\A^{m+k}.$$\\
(Note that $r=1-\sqrt{\frac{2}{\e k}}$, $k=m(\frac{\e r m}{d}-1)$ have a solutions in $\RR^+$ for $m\geq \frac{4d}{\e}$.)\\
We will say a hypothesis space satisfies {\bfseries(M4)}\index{M4}, or is {\bfseries well behaved} if:\\ $$\text{ for all }m\geq 1,\ \e>0,\ C\in\A,\ P\in Pr(X,\A),$$ we have $$Q^m_\e\in\A^m,$$ and $$\text{ for all }m\geq1,\ \e>0,\ k=m,\ r=\frac{m}{2},\ C\in\A,\ P\in Pr(X,\A),$$ we have $$J^{m+k}_{\e,r}\in\A^{m+k}.$$
\end{defn}

\begin{prop}[\cite{Blumer:1989:LVD:76359.76371}]
If $\mathfrak{H}$ is universally separable and (M2) then $\mathfrak{H}$ satisfies (M3) and (M4).
\end{prop}
\begin{proof}
Fix $m\geq 1,\ \e>0,\ C\in\A, \ P\in Pr(X,\A)$. Since $\hh$ is universally separable, $\m{D}_C$ is universally separable. Let $\m{D}'$ be a countable universally dense subset of $\m{D}_C$, let $\{\delta_i\}_{i=1}^{\infty}\s\RR$ be a decreasing sequence converging to zero, let $\{\e_i\}_{i=1}^{\infty}\s\RR$ be a decreasing sequence converging to $\e$, and let
$$\m{D}'_{i,j}=\{D'\in \m{D}': \exists D\in\m{D}_C\text{ with }P(D)\geq\e_i\text{ and } P(D'\triangle D)\leq \delta_j\}$$
We will show
$$Q^m_\e=\bigcup_{i=0}^{\infty}\bigcap_{j=0}^{\infty}\bigcup_{D'\in\m{D}'_{i,j}}(D'^c)^m,$$
and hence $Q^m_\e\in\A^m$.\\
Let $A\in\bigcup_{i=0}^{\infty}\bigcap_{j=0}^{\infty}\bigcup_{D'\in\m{D}'_{i,j}}(D'^c)^m$. Then $A\in(D'^c)^m$ for some $D'\in\m{D}'_{i,j}$ where $\e_i - \delta_i>\e$. This implies $D'\in\m{D}_C^\e$ and so $A\in Q^m_\e$.\\
Now let $A\in Q^m_\e$. Then there is $D\in\m{D}_C^\e$ and $i\in\NN$ such that $A\in (D^c)^m$ and $P(D)\geq \e_i$. Picking any sequence of sets $\{D_i'\}_{i=0}^{\infty}$ in $\m{D}'$ converging pointwise to $D$, we have also have $\{(D_i'^c)^m\}_{i=0}^{\infty}$ converges pointwise to $(D^c)^m$, and so for every $j\in\NN$ there is $D_i'$ where $P(D_i'\triangle D)\leq \delta_j$ and $A\in(D_i'^c)^m$. Thus $A\in\bigcup_{i=0}^{\infty}\bigcap_{j=0}^{\infty}\bigcup_{D'\in\m{D}'_{i,j}}(D'^c)^m$.
Using similar arguments we can show, given $k$ and $r$ such that $J^{m+k}_{\e,r}$ is defined, we have
$$J^{m+k}_{\e,r}=\bigcup_{i=0}^{\infty}\bigcap_{j=0}^{\infty}\bigcup_{D'\in\m{D}'_{i,j}}((D'^c)^m\times X^k\cap\bigcup_{l=\lceil kr\e\rceil}^{k}\{A\in X^k:\text{ exactly }l\text{ coordinates of } A \text{ are in }D'^c\}).$$
\end{proof}

\begin{defn}
Let $\{n_i\}_{i=0}^{k}$ be a finite sequence in $\NN$. We will say that a function 
$$\m{H}:\bigcup_{i\in\{j\in\NN:0\leq j\leq k,\ n_j\neq 0\}}([X]^{= i}\times\{1,...,n_i\})\rightarrow 2^X$$
satisfies {\bfseries (M5)}\index{M5} if for every $1\leq i \leq k$ and $l\in\{1,...,n_i\}$,
$$\{(x_1,...,x_{i+1})\in X^{i+1}:\HH(\{x_1,...,x_i\}\times\min(l,n_{|\{x_1,...,x_i\}|}))(x_{i+1})=1\}\in\A^{i+1}.$$
In particular, a function
$$\m{H}:[X]^{\leq d}\rightarrow 2^X$$
satisfies (M5) if for every $1\leq i \leq d$
$$\{(x_1,...,x_{i+1})\in X^{i+1}:\HH(\{x_1,...,x_i\})(x_{i+1})=1\}\in\A^{i+1}.$$
\end{defn}



\section{Sample Complexity of Copy Sample Compression Schemes}
~~

Here we include the omitted proofs from section 3.2 regarding sample complexity of copy sample compression schemes.
\begin{theo}
Let $P$ be any probability measure on a measurable space $(X,\mathfrak{A})$, $C$ a concept in $\m{C}\s \mathfrak {A}$, $\{n_i\}_{i=0}^{k}\s\NN$, and $\m{H}$ any function from \\
$\bigcup_{i\in\{j\in\NN:0\leq j\leq k,\ n_j\neq 0\}}([X]^{= i}\times\{1,...,n_i\})$ to $2^X$, satisfying measurability condition (M5). Then the probability that $A\s X$, $|A|=m\geq k$, contains a subset $\0$ of size at most $k$, and $l\in\{1,...,n_{|\0|}\}$ such that $P(\M{H}(\0\times l)\triangle C)>\varepsilon >0$ and $\M{H}(\0\times l)_{|_A}=C_{|_A}$, is at most 
$$\sum_{i=0}^k n_i\binom{m}{i}(1-\varepsilon)^{m-i}.$$
\end{theo}
\begin{proof}
Let $C\in\CC$ and $\varepsilon$ be given.
First we consider the probability that a set of size $m$ has a subset of size exactly $i\leq d$ with the property $P(\M{H}(\0\times l)\triangle C)>\varepsilon$ and $\M{H}(\0\times l)_{|_A}=C_{|_A}$ for some $l\in\{1,...,n_{|\0|}\}$. For $A=(a_1,...,a_m)\in X^m$ and $J=\{j_1,...,j_i\}\s\{1,...,m\}$, let $A_{|_J}$ denote $\{a_{j_1},...,a_{j_i}\}$.

There are $\binom{m}{i}$ many subsets of $A$ of size $i$, hence fixing $J$ a subset of $\{1,...,m\}$ of size $i$, the probability we wish to bound above is at most
\begin{align*}
&P^m(\{A\in X^m: \ \exists I\s\{1,...,m\}\text{ of size }i,\ \exists t\in\{1,...,n_i\},\\&~~~~~~~~~~~~~~~~~~\text{ such that } P(\HH(A_{|_I}\times t)\triangle C)>\varepsilon\text{ and }\M{H}(A_{|_I}\times \min(t,n_{|A_{|_I}|}))_{|_A}=C_{|_A} \})\\
&=\sum_{l=1}^{n_i}\binom{m}{i}P^m(\{A\in X^m: \ P(\HH(A_{|_J}\times \min(l,n_{|A_{|_J}|}))\triangle C)>\varepsilon\\ 
&~~~~~~~~~~~~~~~~~~~~~~~~~~~~~~~~~~~~~~~~~~~~~~~~~~~~~~~~~~~~\text{ and }\M{H}(A_{|_J}\times\min(l,n_{|A_{|_J}|}))_{|_A}=C_{|_A} \}). 
\end{align*}
Since permuting $J$ to some other subset of size $i
$ in $\{1,...,m\}$ does not affect the above probability, we can assume $J=\{1,...,i\}$. Fix $l\in\{1,...,n_i\}$.

We will prove at this point that\\
$\{A\in X^m: \ P(\HH(A_{|_J}\times \min(l,n_{|A_{|_J}|}))\triangle C)>\varepsilon\text{ and }\M{H}(A_{|_J}\times\min(l,n_{|A_{|_J}|}))_{|_A}=C_{|_A} \}$ is measurable due to the hypothesis that $\HH$ satisfies (M5): Let $1\leq p<q\leq m$ and let $\pi_{p,q}$ be the (measurable) function from $X^m$ to $X^{p+1}$ mapping $(x_1,...,x_m)\mapsto(x_1,...,x_p,x_q)$. By (M5) and the measurability of $C$ we have
\begin{align*}
&\{A\in X^m: \M{H}(A_{|_J}\times\min(l,n_{|A_{|_J}|}))_{|_A}=C_{|_A} \}\\
&=\{A\in X^m:A\in((\M{H}(A_{|_J}\times\min(l,n_{|A_{|_J}|}))\triangle C)^c)^m\}\\
&=\bigcap_{q=1}^m\{A\in X^m:(\M{H}(A_{|_J}\times\min(l,n_{|A_{|_J}|})))\triangle C)^c)(A_{|_{\{q\}}})=1\}\\
&=\bigcap_{q=1}^m\pi_{i,q}^{-1}(\{(x_1,...,x_{i+1})\in X^{i+1}:(\M{H}(\{x_1,...,x_i\}\times\min(l,n_{|\{x_1,...,x_i\}|}))\triangle C)^c)(x_{i+1})=1\})\\
&\in \A^m.
\end{align*}
Also $\{A\in X^m: P(\HH(A_{|_J}\times\min(l,n_{|A_{|_J}|}))\triangle C)>\varepsilon\}$ is measurable since 
\begin{align*}
&B=\{(x_1,...,x_{m+1})\in X^{m+1}:(\M{H}(\{x_1,...,x_{i}\}\times\min(l,n_{|\{x_1,...,x_i\}|}))\triangle C)^c(x_{m+1})=1\}\\
&=\pi_{i,m+1}^{-1}(\{(x_1,...,x_{i+1})\in X^{i+1}:(\M{H}(\{x_1,...,x_{i}\}\times\min(l,n_{|\{x_1,...,x_i\}|}))\triangle C)^c(x_{i+1})=1\})
\end{align*}
is measurable by (M5) and the measurability of $C$, and a straightforward application of Fubini's theorem gives us that the map
\begin{align*}
&(x_1,...,x_m)\mapsto \int_X\chi_B(x_1,...,x_{m+1})dP(x_{m+1})=\\
&=P(\{y:(x_1,...,x_m,y)\in B\})\\
&=P(\{y:y\in\M{H}(\{x_1,...,x_{i}\}\times\min(l,n_{|\{x_1,...,x_i\}|}))\triangle C)^c\})\\
&=P(\M{H}(\{x_1,...,x_{i}\}\times\min(l,n_{|\{x_1,...,x_i\}|}))\triangle C)^c)
\end{align*}
is measurable.

Now let 
\begin{align*}
&E_C=\{A\in X^{m}:\M{H}(A_{|_J}\times\min(l,n_{|A_{|_J}|}))_{|_{A_{|_{\{i+1,...,m\}}}}}=C_{|_{A_{|_{\{i+1,...,m\}}}}}\}\\
&=\{A\in X^m:A_{|_{\{i+1,...,m\}}}\in((\M{H}(A_{|_J}\times\min(l,n_{|A_{|_J}|}))\triangle C)^c)^{m-i}\}\\
&=\bigcap_{q=i+1}^m\{A\in X^m:(\M{H}(A_{|_J}\times\min(l,n_{|A_{|_J}|})))\triangle C)^c)(A_{|_{\{q\}}})=1\}\\
&=\bigcap_{q=i+1}^m\pi_{i,q}^{-1}(\{(x_1,...,x_{i+1})\in X^{i+1}:(\M{H}(\{x_1,...,x_i\}\times\min(l,n_{|\{x_1,...,x_i\}|}))\triangle C)^c)(x_{i+1})=1\})\\
&\in \A^m.
\end{align*} 
and 
\begin{align*}
E_\varepsilon&=\{A\in X^i:P(\M{H}(A\times\min(l,n_{|A_{|_J}|}))\triangle C)>\varepsilon\}\\
&=\{A\in X^i:P((\M{H}(A\times\min(l,n_{|A_{|_J}|}))\triangle C)^c)\leq (1-\varepsilon)\}.
\end{align*} 
$E_\varepsilon$ is measurable since \\$B=\{(x_1,...,x_{i+1})\in X^{i+1}:(\M{H}(\{x_1,...,x_{i}\}\times\min(l,n_{|\{x_1,...,x_i\}|}))\triangle C)^c(x_{i+1})=1\}$ is measurable by (M5) and the measurability of $C$, and a straightforward application of Fubini's theorem gives us that the map
\begin{align*}
&(x_1,...,x_i)\mapsto \int_X\chi_B(x_1,...,x_{i+1})dP(x_{i+1})=\\
&=P(\{y:(x_1,...,x_i,y)\in B\})\\
&=P(\{y:y\in\M{H}(\{x_1,...,x_{i}\}\times\min(l,n_{|\{x_1,...,x_i\}|}))\triangle C)^c\})\\
&=P(\M{H}(\{x_1,...,x_{i}\}\times\min(l,n_{|\{x_1,...,x_i\}|}))\triangle C)^c)
\end{align*}
is measurable.\\
We have that
\begin{align*}
&P^m(\{A\in X^m: P(\HH(A_{|_J}\times\min(l,n_{|A_{|_J}|}))\triangle C)>\varepsilon\\
&~~~~~~~~~~~~~~~~~~~~~~~~~~~~~~~~~~~~~~~~~~~~~~~~~~~~~~~~~~~~~\text{ and }\M{H}(A_{|_J}\times\min(l,n_{|A_{|_J}|}))_{|_A}=C_{|_A} \})\\
&\leq P^m(\{A\in X^m: P(\HH(A_{|_J}\times\min(l,n_{|A_{|_J}|}))\triangle C)>\varepsilon\\
&~~~~~~~~~~~~~~~~~~~~~~~~~~~~~~~~~~~~~~~~\text{ and }\M{H}(A_{|_J}\times\min(l,n_{|A_{|_J}|}))_{|_{A_{|_{\{i+1,...,m\}}}}}=C_{|_{A_{|_{\{i+1,...,m\}}}}} \})\\
&= P^m(E_C\cap (E_\e\times X^{m-i})).
\end{align*}
By Fubini's theorem
\begin{align*}
&P^m(E_C\cap (E_\e\times X^{m-i}))=\int_{E_\e\times X^{m-i}}\chi_{E_C}(x_1,...,x_m)dP^m\\
&=\int_{E_\e}\left(\int_{X^{m-i}}\chi_{E_C}(x_1,...,x_m)dP^{m-i}\right)dP^i.
\end{align*}
Now 
\begin{align*}
&(x_1,...,x_i)\times X^{m-i}\cap E_C\\
&=(x_1,...,x_i)\times\{A\in X^{m-i}:A_{|_{\{i+1,...,m\}}}\in((\M{H}(A_{|_J}\times\min(l,n_{|A_{|_J}|}))\triangle C)^c)^{m-i}\}
\end{align*}
and since $(x_1,...,x_i)\in E_\e$, the inner integral is at most $(1-\e)^{m-i}$ and so 
$$P^m(E_C\cap (E_\e\times X^{m-i}))\leq(1-\e)^{m-i}.$$ 
Therefore summing over all subsets $J$ of $\{1,...,m\}$ of size at most $k$, the probability that $A\s X$, $|A|=m\geq k$, contains a subset $\0$ of size at most $k$, and $l\in\{1,...,n_{|\0|}\}$, such that $P(\M{H}(\0\times l)\triangle C)>\varepsilon >0$ and $\M{H}(\0\times l)_{|_A}=C_{|_A}$, is at most $$\sum_{i=0}^k\sum_{l=1}^{n_i}\binom{m}{i}(1-\varepsilon)^{m-i}=\sum_{i=0}^kn_i\binom{m}{i}(1-\varepsilon)^{m-i}.$$
\end{proof}

\begin{lem}
Let $0<\e\leq1,\ 0<\delta\leq1$, $m,k$ positive integers, and $n=\max(\{n_i\}_{i=0}^{k})$. If there is $0<\beta<1$ where 
$$m\geq\frac{1}{1-\beta}\left(\frac{1}{\e}\ln(\frac{n}{\delta})+k+\frac{k}{\e}\ln(\frac{1}{\beta\e})\right),$$ then 
$$\sum_{i=0}^kn_i\binom{m}{i}(1-\varepsilon)^{m-i}\leq\delta.$$
\end{lem}
\begin{proof}
Let $\e,\delta,\beta,m,k,n,\{n_i\}_{i=0}^k$ be as in the statement of the lemma. Then 
$$\frac{1}{\e}\ln(\frac{n}{\delta})+k+\frac{k}{\e}\left(-\ln(k)+\ln(\frac{k}{\beta\e})-1+\frac{\beta\e}{k}m+1\right)$$
$$=\beta m + \frac{1}{\e}\ln(\frac{n}{\delta})+k+\frac{k}{\e}\ln(\frac{1}{\beta\e})\leq m.$$
We will use the fact that $\ln(m)\leq-\ln(\alpha)-1+\alpha m$ for all $\alpha>0$. 
With $\alpha=\frac{\beta\e}{k}$\\we get $$\ln(m)\leq\ln(\frac{k}{\beta\e})-1+\frac{\beta\e}{k}m$$ thus 
$$\frac{1}{\e}\ln(\frac{n}{\delta})+k+\frac{k}{\e}\left(-\ln(k)+\ln(m)+1\right)$$
$$\leq\frac{1}{\e}\ln(\frac{n}{\delta})+k+\frac{k}{\e}\left(-\ln(k)+\ln(\frac{k}{\beta\e})-1+\frac{\beta\e}{k}m+1\right)\leq m.$$ Hence we have $$\ln(\frac{n}{\delta})+k(-\ln(k)+\ln(m)+1)\leq \e(m-k)$$ and so $$n\left(\frac{em}{k}\right)^k\leq e^{\e(m-k)}\delta.$$ Therefore since $m\geq k$, 
$$\sum_{i=0}^kn_i\binom{m}{i}(1-\varepsilon)^{m-i}\leq n\binom{m}{\leq k}(1-\varepsilon)^{m-k}\leq n\left(\frac{em}{k}\right)^k(1-\varepsilon)^{m-k}\leq n\left(\frac{em}{k}\right)^ke^{-\e(m-k)}$$
$$\leq\delta.$$
\end{proof}

\begin{theo}
If $(X,\CC)$ has a $\{n_i\}_{i=0}^{k}$-copy sample compression scheme $\HH$ of size $k$ satisfying (M5), and $n=\max(\{n_i\}_{i=0}^{d})$, then for $0<\e\leq1,\ 0<\delta\leq1$, if $\LL_{\HH}$ is a learning rule (if $\LL_{\HH}$ satisfies (M2)) it has sample complexity at most
$$m_{\LL_{\HH}}(\e,\delta)\leq\frac{1}{1-\beta}\left(\frac{1}{\e}\ln(\frac{n}{\delta})+k+\frac{k}{\e}\ln(\frac{1}{\beta\e})\right),$$
for any $0<\beta<1$.
\end{theo}
\begin{proof}
Let $\HH$ be as in the statement of the theorem, fix $0<\e\leq1,\ 0<\delta\leq1,\ 0<\beta<1,\ C\in\CC,\ P\in \m{P}$, and let 
$$m\geq\frac{1}{1-\beta}\left(\frac{1}{\e}\ln(\frac{n}{\delta})+k+\frac{k}{\e}\ln(\frac{1}{\beta\e})\right).$$
Since $\HH$ is consistent and satisfies (M5), by the previous theorem
\begin{align*}
&P^m(\{A\in X^m:P(\LL_{\HH}(A,C_{|_A})\triangle C)>\e\})\\
&\leq P^m(\{A\in X^m:\exists \0\in[A]^{\leq d},\ \exists l\in\{1,...,n_{|\0|}\} \text{ where } P(\M{H}(\0\times\min(l,n_{|\0|}))\triangle C)>\varepsilon\})\\
&=P^m(\{A\in X^m:\exists \0\in[A]^{\leq d},\ \exists l\in\{1,...,n_{|\0|}\} \text{ where } P(\M{H}(\0\times\min(l,n_{|\0|}))\triangle C)>\varepsilon\\
&~~~~~~~~~~~~~~~~~~~~~~~~~~~~~~~~~~~~~~~~~~~~~~~~~~~~~~~~~~~~~~~~~~~~~\text{ and }\M{H}(\0\times\min(l,n_{|\0|}))_{|_A}=C_{|_A}\})\\
&\leq\sum_{i=0}^k n_i\binom{m}{i}(1-\varepsilon)^{m-i},
\end{align*}
and by the previous lemma
$$\sum_{i=0}^k n_i\binom{m}{i}(1-\varepsilon)^{m-i}\leq\delta.$$
\end{proof}
\cleardoublepage

%
%
%
%
%
%
\printindex

\bibTexTest{biblio}         

\end{document}